\documentclass[11pt]{article}
\pdfoutput=1

\usepackage{amsmath,amssymb}
\usepackage{bm}
\usepackage{natbib}
\usepackage[usenames]{color}
\usepackage{amsthm}

\usepackage{multirow} 
\usepackage{enumitem}

\usepackage[colorlinks,
linkcolor=red,
anchorcolor=blue,
citecolor=blue
]{hyperref}

\usepackage{setspace}
\usepackage[left=1in, right=1in, top=1in, bottom=1in]{geometry}

\usepackage{xcolor}

\usepackage{smile}

\title{\huge Generalization Bounds of Stochastic Gradient Descent for Wide and Deep Neural Networks}

\author
{
	Yuan Cao\thanks{Department of Computer Science, University of California, Los Angeles, CA 90095, USA; e-mail: {\tt yuancao@cs.ucla.edu}} 
	~~~and~~~
	Quanquan Gu\thanks{Department of Computer Science, University of California, Los Angeles, CA 90095, USA; e-mail: {\tt qgu@cs.ucla.edu}}
}
\date{}

\def\supp{\mathrm{supp}}

\def\Tr{\mathrm{Tr}}
\def\poly{\mathrm{poly}}

\newcommand{\la}{\langle}
\newcommand{\ra}{\rangle}

\begin{document}

\maketitle

\begin{abstract}
We study the training and generalization of deep neural networks (DNNs) in the over-parameterized regime, where the network  width (i.e., number of hidden nodes per layer) is much larger than the number of training data points. We show that, the expected $0$-$1$ loss of a wide enough ReLU network trained with stochastic gradient descent (SGD) and random initialization can be bounded by the training loss of a random feature model induced by the network gradient at initialization, which we call a \textit{neural tangent random feature} (NTRF) model. For data distributions that can be classified by NTRF model with sufficiently small error, our result yields a generalization error bound in the order of $\tilde\cO(n^{-1/2})$ that is independent of the network width. Our result is more general and sharper than many existing generalization error bounds for over-parameterized neural networks. In addition, we establish a strong connection between our generalization error bound and the neural tangent kernel (NTK) proposed in recent work. 
\end{abstract}



\section{Introduction}

Deep learning has achieved great success in a wide range of applications including image processing \citep{krizhevsky2012imagenet}, natural language processing \citep{hinton2012deep} and reinforcement learning \citep{silver2016mastering}. 
Most of the deep neural networks used in practice are highly over-parameterized, such that the number of parameters is much larger than the number of training data. 
One of the mysteries in deep learning is that, even in an over-parameterized regime, neural networks trained with stochastic gradient descent can still give small test error and do not overfit. In fact, a famous empirical study by \citet{zhang2016understanding} shows the following phenomena:
\begin{itemize}[leftmargin = *]
    \item Even if one replaces the real labels of a training data set with purely random labels, an over-parameterized neural network can still fit the training data perfectly. However since the labels are independent of the input, the resulting neural network does not generalize to the test dataset.
    \item If the same over-parameterized network is trained with real labels, it not only achieves small training loss, but also generalizes well to the test dataset.
\end{itemize}
While a series of recent work has theoretically shown that a sufficiently over-parameterized (i.e., sufficiently wide) neural network can fit random labels \citep{du2018gradient,allen2018convergence,du2018gradientdeep,zou2018stochastic}, the reason why it can generalize well when trained with real labels is less understood. Existing generalization bounds for deep neural networks \citep{neyshabur2015norm,bartlett2017spectrally,neyshabur2017pac,golowich2017size,dziugaite2017computing, arora2018stronger,li2018tighter,wei2018margin,neyshabur2018role} based on uniform convergence usually cannot provide non-vacuous bounds \citep{langford2002not,dziugaite2017computing} in the over-parameterized regime. In fact, the empirical observation by \citet{zhang2016understanding} indicates that in order to understand deep learning, 
it is important to distinguish the true data labels from random labels when studying generalization. In other words, it is essential to quantify the ``classifiability'' of the underlying data distribution, i.e., how difficult it can be classified.  

Certain effort has been made  to take the ``classifiability'' of the data distribution into account for generalization analysis of neural networks. 
\citet{brutzkus2017sgd} showed that stochastic gradient descent (SGD) can learn an over-parameterized two-layer neural network with good generalization for linearly separable data. \citet{li2018learning} proved that, if the data satisfy certain structural assumption, SGD can learn an over-parameterized two-layer network with fixed second layer weights and achieve a small generalization error. \citet{allen2018generalization} studied the generalization performance of SGD and its variants for learning two-layer and three-layer networks, and used  the risk of smaller two-layer or three-layer networks with smooth activation functions to characterize the classifiability of the data distribution.
There is another line of studies on the algorithm-dependent generalization bounds of neural networks in the over-parameterized regime  \citep{daniely2017sgd,arora2019fine,cao2019generalization,yehudai2019power,e2019comparative}, which
quantifies the classifiability of the data
with a reference function class defined by random features \citep{rahimi2008random,rahimi2009weighted} or kernels\footnote{Since random feature models and kernel methods are highly related \citep{rahimi2008random,rahimi2009weighted}, we group them into the same category. More details are discussed in Section~\ref{subsection:connection_to_NTK}.}. 
Specifically, \citet{daniely2017sgd} showed that a neural network of large enough size is competitive with the best function in the conjugate kernel class of the network. \citet{arora2019fine} gave a generalization error bound for two-layer ReLU networks with fixed second layer weights based on a ReLU kernel function. \citet{cao2019generalization} showed that deep ReLU networks trained with gradient descent can achieve small generalization error if the data can be separated by certain random feature model \citep{rahimi2009weighted} with a margin. \citet{yehudai2019power} used the expected loss of a similar random feature model to quantify the generalization error of two-layer neural networks with smooth activation functions. A similar generalization error bound was also given by \citet{e2019comparative}, where the authors studied the optimization and generalization of two-layer networks trained with gradient descent. However, all the aforementioned results are still far from satisfactory: they are either limited to two-layer networks, or restricted to very simple and special reference function classes.

In this paper, we aim at providing a sharper and generic analysis on the generalization of deep ReLU networks trained by SGD. In detail, we base our analysis upon the key observations that near random initialization, the neural network function is almost a linear function of its parameters and the loss function is locally almost convex. This enables us to prove a cumulative loss bound of SGD, which further leads to a generalization bound by online-to-batch conversion \citep{cesa2004generalization}. 
The main contributions of our work are summarized as follows:
\begin{itemize}[leftmargin= *]
    \item We give a bound on the expected $0$-$1$ error of deep ReLU networks trained by SGD with random initialization. Our result relates the generalization bound of an over-parameterized ReLU network with a random feature model defined by the network gradients, which we call \textit{neural tangent random feature} (NTRF) model. It also suggests an algorithm-dependent generalization error bound of order $\tilde\cO(n^{-1/2})$, which is independent of network width, if the data can be classified by the NTRF model with small enough error. 
    \item Our analysis is general enough to cover recent generalization error bounds for neural networks with random feature based reference function classes, and provides better bounds. Our expected $0$-$1$ error bound directly covers the result by \citet{cao2019generalization}, and gives a tighter sample complexity when reduced to their setting, i.e., $\tilde\cO(1/\epsilon^2)$ versus $\tilde\cO(1/\epsilon^4)$ where $\epsilon$ is the target generalization error. Compared with recent results by  \citet{yehudai2019power,e2019comparative} who only studied two-layer networks, our bound not only works for deep networks, but also uses a larger reference function class when reduced to the two-layer setting, and therefore is  sharper. 
    \item Our result has a direct connection to the neural tangent kernel studied in \citet{jacot2018neural}. When interpreted in the language of kernel method, our result gives a generalization bound in the form of $\tilde\cO(L\cdot \sqrt{\yb^\top (\bTheta^{(L)})^{-1} \yb / n })$, where $\yb$ is the training label vector, and $\bTheta^{(L)}$ is the neural tangent kernel matrix defined on the training input data. This form of generalization bound is similar to, but more general and tighter than the bound given by \citet{arora2019fine}.   
\end{itemize}

\noindent\textbf{Notation}
We use lower case, lower case bold face, and upper case bold face letters to denote scalars, vectors and matrices respectively. 
For a vector $\vb=(v_1,\ldots,v_d)^T\in \RR^d$ and a number $1\leq p < \infty$, let $\|\vb\|_p = (\sum_{i=1}^d |v_i|^p)^{1/p}$. We also define $\|\vb\|_\infty = \max_i|v_i|$. For a matrix $\Ab = (A_{i,j})_{m\times n} $, we use $\|\Ab \|_{0}$ to denote the number of non-zero entries of $\Ab$, and denote 
$\| \Ab  \|_F = (\sum_{i,j=1}^d A_{i,j}^2)^{1/2}$ and $\|\Ab \|_p = \max_{\|\vb\|_p = 1} \|\Ab \vb\|_p$ for $p\geq 1$. For two matrices $\Ab,\Bb \in \RR^{m\times n}$, we define $\la \Ab, \Bb \ra = \Tr(\Ab^\top \Bb)$. We denote by $\Ab \succeq \Bb$ if $\Ab - \Bb$ is positive semidefinite. In addition, we define the asymptotic notations $\cO(\cdot)$, $\tilde{\cO}(\cdot)$, $\Omega(\cdot)$ and $\tilde\Omega(\cdot)$ as follows. Suppose that $a_n$ and $b_n$ be two sequences. 
We write $a_n = \cO(b_n)$ if $\limsup_{n\rightarrow \infty} |a_n/b_n| < \infty$, and $a_n = \Omega(b_n)$ if $\liminf_{n\rightarrow \infty} |a_n/b_n| > 0$. We use $\tilde{\cO}(\cdot)$ and $\tilde{\Omega}(\cdot)$ to hide the logarithmic factors in $\cO(\cdot)$ and $\Omega(\cdot)$.


\section{Problem Setup}
In this section we introduce the basic problem setup. 
Following the same standard setup implemented in the line of recent work \citep{allen2018convergence,du2018gradientdeep,zou2018stochastic,cao2019generalization}, we consider fully connected neural networks with width $m$, depth $L$ and input dimension $d$. 
Such a network is defined by its weight matrices at each layer: for $L\geq 2$, let $\Wb_1 \in \RR^{m\times d}$, $\Wb_{l}\in \RR^{m\times m}$, $l=2,\ldots, L-1$ and $\Wb_{L} \in \RR^{1\times m}$ be the weight matrices of the network. Then the neural network with input $\xb\in\RR^d$ is defined as
\begin{align}\label{eq:NNdefinition}
    f_\Wb(\xb) = \sqrt{m}\cdot \Wb_L \sigma( \Wb_{L - 1} \sigma( \Wb_{L - 2} \cdots \sigma( \Wb_{1} \xb ) \cdots )),
\end{align}
where $\sigma(\cdot)$ is the entry-wise activation function. In this paper, we only consider the ReLU activation function $\sigma(z) = \max\{ 0, z \}$, which is the most commonly used activation function in applications. It is also arguably one of the most difficult activation functions to analyze, due to its non-smoothess. We remark that our result can be generalized to many other Lipschitz continuous and smooth activation functions. For simplicity, we follow \citet{allen2018convergence,du2018gradientdeep} and assume that the widths of each hidden layer are the same. Our result can be easily extended to the setting that the widths of each layer are not equal but in the same order, as discussed in \citet{zou2018stochastic,cao2019generalization}. 


When $L=1$, the neural network reduces to a linear function, which has been well-studied. Therefore, for notational simplicity we focus on the case $L\geq 2$, where the parameter space is defined as
\begin{align*}
    \cW := \RR^{m\times d} \times (\RR^{m\times m})^{L- 2} \times \RR^{1\times m}.
\end{align*}
We also use $\Wb = (\Wb_1,\ldots,\Wb_L) \in \cW$ to denote the collection of weight matrices for all layers. 
For $\Wb, \Wb'\in \cW$, we define their inner product as $\la \Wb , \Wb' \ra := \sum_{l=1}^L \Tr ( \Wb_l^\top \Wb_l' )$. 

The goal of neural network learning is to minimize the expected risk, i.e.,
\begin{align}\label{eq:risk}
\min_{\Wb} L_\cD(\Wb):= \EE_{(\xb,y)\sim \cD} L_{(\xb,y)}(\Wb),
\end{align}
where $L_{(\xb,y)}(\Wb) = \ell[ y \cdot f_\Wb(\xb) ] $ is the loss defined on any  example $(\xb,y)$, and $\ell(z)$ is the loss function. Without loss of generality, we consider the cross-entropy loss in this paper, which is defined as  $\ell(z) = \log[1 + \exp(-z)]$. We would like to emphasize that our results also hold for most convex and Lipschitz continuous loss functions such as hinge loss.
We now introduce stochastic gradient descent based training algorithm for minimizing the expected risk in \eqref{eq:risk}. The detailed algorithm is given in Algorithm~\ref{alg:SGDrandominit}. 




\begin{algorithm}[H]
\caption{SGD for DNNs starting at Gaussian initialization}\label{alg:SGDrandominit}
\begin{algorithmic}
\STATE \textbf{Input:} Number of iterations $n$, step size $\eta$.
\STATE Generate each entry of $\Wb_l^{(1)}$ independently from $N(0,2/m)$, $l\in[L-1]$.
\STATE Generate each entry of $\Wb_L^{(1)}$ independently from $N(0,1/m)$.
\FOR{$i=1,2,\ldots, n$}
\STATE Draw $(\xb_i,y_i)$ from $\cD$.
\STATE Update $\Wb^{(i+1)} = \Wb^{(i)} - \eta\cdot \nabla_{\Wb} L_{(\xb_i,y_i)}(\Wb^{(i)})$.
\ENDFOR
\STATE \textbf{Output:} Randomly choose $\hat \Wb$ uniformly from $\{ \Wb^{(1)} ,\ldots, \Wb^{(n)} \}$.
\end{algorithmic}
\end{algorithm}

The initialization scheme for $\Wb^{(1)}$ given in Algorithm~\ref{alg:SGDrandominit} generates each entry of the weight matrices from a zero-mean independent Gaussian distribution, whose variance is determined by the rule that the expected length of the output vector in each hidden layer  is equal to the length of the input. This initialization method is also known as He initialization \citep{he2015delving}. Here the last layer parameter is initialized with variance $1/m$ instead of $2/m$ since the last layer is not associated with the ReLU activation function. 


\section{Main Results}
In this section we present the main results of this paper. In Section~\ref{subsection:expected01bound} we give an expected $0$-$1$ error bound against a neural tangent random feature reference function class. In Section~\ref{subsection:connection_to_NTK}, we discuss the connection between our result and the neural tangent kernel proposed in \citet{jacot2018neural}. 
\subsection{An Expected $0$-$1$ Error Bound}\label{subsection:expected01bound}
In this section we 
give a bound on the expected $0$-$1$ error $L_\cD^{0-1}(\Wb) := \EE_{(\xb,y)\sim \cD} [\ind \{y\cdot f_{\Wb}(\xb)<0\}]$ obtained by Algorithm~\ref{alg:SGDrandominit}. 
Our result is based on the following assumption.

\begin{assumption}\label{assump:normalizeddata}
The data inputs are normalized: $\| \xb \|_2 = 1$ for all $(\xb,y) \in \supp(\cD)$.
\end{assumption}

Assumption~\ref{assump:normalizeddata} is a standard assumption made in almost all previous work on optimization and generalization of over-parameterized neural networks \citep{du2018gradient,allen2018convergence,du2018gradientdeep,zou2018stochastic,oymak2019towards,e2019comparative}. As is mentioned in \citet{cao2019generalization}, this assumption can be relaxed to $c_1 \leq \| \xb \|_2 \leq c_2$ for all $(\xb,y)\in \supp(\cD)$, where $c_2 > c_1 >0$ are absolute constants. 


For any $\Wb\in \cW$, we define its $\omega$-neighborhood as 
\begin{align*}
    \cB(\Wb,\omega) := \{ \Wb'\in \cW: \| \Wb_l' - \Wb_l \|_F \leq \omega, l\in [L] \}.
\end{align*}
Below we introduce the neural tangent random feature function class, which serves as a reference function class to measure the ``classifiability'' of the data, i.e., how easy it can be classified.


\begin{definition}[Neural Tangent Random Feature]\label{def:NTRF}
Let $\Wb^{(1)}$ be generated via the initialization scheme in Algorithm~\ref{alg:SGDrandominit}. The neural tangent random feature (NTRF) function class is defined as
\begin{align*}
    \cF(\Wb^{(1)}, R) = \big\{ f(\cdot) = f_{\Wb^{(1)}}(\cdot) + \la \nabla_\Wb f_{\Wb^{(1)}}(\cdot) , \Wb \ra : \Wb\in \cB(\mathbf{0} , R \cdot m^{-1/2} ) \big\},
\end{align*}
where $R>0$ measures the size of the function class, and $m$ is the width of the neural network.
\end{definition}
The name ``neural tangent random feature'' is inspired by the neural tangent kernel proposed by \citet{jacot2018neural}, because the random features are the gradients of the neural network with random weights. Connections between the neural tangent random features and the neural tangent kernel will be discussed in Section~\ref{subsection:connection_to_NTK}.


We are ready to present our main result on the expected $0$-$1$ error bound of Algorithm~\ref{alg:SGDrandominit}. 

\begin{theorem}\label{thm:expectederrorbound}
    For any $\delta \in (0,e^{-1}]$ and $R > 0$, there exists 
    \begin{align*}
        m^* (\delta, R,L,n) = \tilde\cO\big( \poly(R,L) \big)\cdot  n^{7} \cdot \log(1 / \delta)
    \end{align*}
    such that if $m \geq m^* (\delta, R,L,n)$, then with probability at least $1 - \delta$ over the randomness of $\Wb^{(1)}$, 
    the output of Algorithm~\ref{alg:SGDrandominit} with step size $\eta = \kappa\cdot R / (m\sqrt{n}) $ 
    for some small enough absolute constant $\kappa$ satisfies
    \begin{align}\label{eq:result_expectederrorbound}
        \EE \big[ L_{\cD}^{0-1}( \hat\Wb ) \big] \leq \inf_{ f \in \cF( \Wb^{(1)}, R )} \Bigg\{ \frac{4}{n}\sum_{i=1}^n \ell[y_i\cdot f(\xb_i) ] \Bigg\} + \cO\Bigg[ \frac{LR}{\sqrt{n}} + \sqrt{\frac{\log(1 / \delta)}{n}} \Bigg],
    \end{align}
    where the expectation is taken over 
    the uniform draw of $\hat\Wb$ from $\{\Wb^{(1)},\ldots, \Wb^{(n)}\}$. 
\end{theorem}

The expected $0$-$1$ error bound given by Theorem~\ref{thm:expectederrorbound} consists of two terms:
    The first term in \eqref{eq:result_expectederrorbound} relates the expected $0$-$1$ error achieved by Algorithm~\ref{alg:SGDrandominit} with a reference function class--the NTRF function class in Definition \ref{def:NTRF}. 
    The second term in \eqref{eq:result_expectederrorbound} is a standard large-deviation error term. As long as $R=\tilde\cO(1)$, 
    this term matches the standard $\tilde\cO(n^{-1/2})$ rate in PAC learning bounds \citep{shalev2014understanding}.

\begin{remark}
The parameter $R$ in Theorem~\ref{thm:expectederrorbound} is from the NTRF class and introduces a trade-off in the bound: when $R$ is small, the corresponding NTRF class $\cF(\Wb^{(1)},R)$ is small, making the first term in \eqref{eq:result_expectederrorbound} large, and the second term in \eqref{eq:result_expectederrorbound} is small. When $R$ is large, the corresponding function class $\cF(\Wb^{(1)},R)$ is large, so the first term in \eqref{eq:result_expectederrorbound} is small, and the second term will be large. In particular, if we set $R=\tilde\cO(1)$, the second term in \eqref{eq:result_expectederrorbound} will be $\tilde\cO(n^{-1/2})$. In this case, the ``classifiability'' of the underlying data distribution $\cD$ is determined by how well its i.i.d. samples can be classified by $\cF(\Wb^{(1)},\tilde\cO(1))$.
In other words, Theorem~\ref{thm:expectederrorbound} suggests that if the data can be classified by a function in the NTRF function class $\cF(\Wb^{(1)},\tilde\cO(1))$ with a small training error, the over-parameterized ReLU network learnt by Algorithm~\ref{alg:SGDrandominit} will have a small generalization error. 
\end{remark}


\begin{remark}\label{remark:comparisiontocao}
The expected $0$-$1$ error bound given by Theorem~\ref{thm:expectederrorbound} is in a very general form. 
It directly covers the result given by \citet{cao2019generalization}. In Appendix~\ref{section:comparisontocao}, we show that under the same assumptions made in \citet{cao2019generalization}, to achieve $\epsilon$ expected $0$-$1$ error, our result requires a sample complexity of order $\tilde\cO(\epsilon^{-2})$, which outperforms the 
result in \citet{cao2019generalization} by a factor of $\epsilon^{-2}$. 
\end{remark}

\begin{remark}\label{remark:comparisiontoyehudai&e}
Our generalization bound can also be compared with two recent results \citep{yehudai2019power,e2019comparative} for two-layer neural networks. 
When $L=2$, the NTRF function class $\cF(\Wb^{(1)},\tilde\cO(1))$ can be written as 
\begin{align*}
    \big\{ f_{\Wb^{(1)}}(\cdot) + \la \nabla_{\Wb_1} f_{\Wb^{(1)}}(\cdot) , \Wb_1 \ra + \la \nabla_{\Wb_2} f_{\Wb^{(1)}}(\cdot) , \Wb_2 \ra : \|\Wb_1\|_F,\|\Wb_2\|_F \leq \tilde\cO(m^{-1/2})  \big\}.
\end{align*}
In contrast, the reference function classes studied by \citet{yehudai2019power} and \citet{e2019comparative} are contained in the following random feature class: 
\begin{align*}
    \cF = \big\{ f_{\Wb^{(1)}}(\cdot) + \la \nabla_{\Wb_2} f_{\Wb^{(1)}}(\cdot) , \Wb_2 \ra : \|\Wb_2\|_F \leq \tilde\cO(m^{-1/2}) \big\},
\end{align*}
where $\Wb^{(1)} = (\Wb^{(1)}_1, \Wb^{(1)}_2) \in \RR^{m\times d}\times \RR^{1\times m}$ are the random weights generated by the
initialization schemes in \citet{yehudai2019power,e2019comparative}\footnote{Normalizing weights to the same scale is necessary for a proper comparison. See Appendix~\ref{section:comparisontoyehudai&e} for details.}. 
Evidently, our NTRF function class $\cF(\Wb^{(1)},\tilde\cO(1))$ is richer than $\cF$--it also contains the features corresponding to the first layer gradient of the network at random initialization, i.e., $\nabla_{\Wb_1}f_{\Wb^{(1)}}(\cdot)$. As a result, our generalization bound is sharper than those in \citet{yehudai2019power,e2019comparative} in the sense that we can show that neural networks trained with SGD can compete with the best function in a larger reference function class. 
\end{remark}

As previously mentioned, the result of Theorem~\ref{thm:expectederrorbound} can be easily extended to the setting where the widths of different layers are different. We should expect that the result remains almost the same, except that we assume the widths of hidden layers are all larger than or equal to $m^* (\delta, R,L,n)$. We would also like to point out that although this paper considers the cross-entropy loss, the proof of Theorem~\ref{thm:expectederrorbound} offers a general framework based on the fact that near initialization, the neural network function is almost linear in terms of its weights. We believe that this proof framework can potentially be applied to most practically useful loss functions: whenever $\ell(\cdot)$ is convex/Lipschitz continuous/smooth, near initialization, $L_i(\Wb)$ is also almost convex/Lipschitz continuous/smooth in $\Wb$ for all $i\in [n]$, and therefore standard online optimization analysis can be invoked with online-to-batch conversion to provide a generalization bound. We refer to Section~\ref{section:proof_in_mainpaper} for more details. 

\subsection{Connection to Neural Tangent Kernel}\label{subsection:connection_to_NTK}

Besides quantifying the classifiability of the data with the NTRF function class $\cF(\Wb^{(1)},\tilde\cO(1))$, an alternative way to apply Theorem~\ref{thm:expectederrorbound} is to check how large the parameter $R$ needs to be in order to make the first term in \eqref{eq:result_expectederrorbound} small enough (e.g., smaller than $n^{-1/2}$). 
In this subsection, we show that this type of analysis connects Theorem~\ref{thm:expectederrorbound} to the neural tangent kernel proposed in \citet{jacot2018neural} and later studied by \citet{yang2019scaling,lee2019wide,arora2019exact}. 
Specifically, we provide an expected $0$-$1$ error bound in terms of the neural tangent kernel matrix defined over the training data.
We first define the neural tangent kernel matrix for the neural network function in \eqref{eq:NNdefinition}. 

\begin{definition}[Neural Tangent Kernel Matrix]\label{def:nueraltangentkernel}
For any $i,j \in [n]$, define
\begin{align*}
    &\tilde\bTheta_{i,j}^{(1)} = \bSigma_{i,j}^{(1)} = \la \xb_i, \xb_j \ra,\quad
    \Ab_{ij}^{(l)} = \begin{pmatrix} 
\bSigma_{i,i}^{(l)} & \bSigma_{i,j}^{(l)} \\
\bSigma_{i,j}^{(l)} & \bSigma_{j,j}^{(l)} 
\end{pmatrix},\\
    &\bSigma_{i,j}^{(l+1)} = 2\cdot \EE_{(u,v)\sim N\big(\mathbf{0}, \Ab_{ij}^{(l)}\big)} [\sigma(u) \sigma(v) ],\\
    &\tilde\bTheta_{i,j}^{(l+1)} = \tilde\bTheta_{i,j}^{(l)}\cdot 2 \cdot \EE_{(u,v)\sim N\big(\mathbf{0}, \Ab_{ij}^{(l)}\big)} [\sigma'(u) \sigma'(v) ] + \bSigma_{i,j}^{(l+1)}.
\end{align*}
Then we call $\bTheta^{(L)} = [(\tilde\bTheta_{i,j}^{(L)} + \bSigma_{i,j}^{(L)})/2]_{n\times n}$ the neural tangent kernel matrix of an $L$-layer ReLU network on training inputs $\xb_1,\ldots,\xb_n$.
\end{definition}


Definition~\ref{def:nueraltangentkernel} is the same as the original definition in \citet{jacot2018neural} when restricting the kernel function on $\{\xb_1,\ldots,\xb_n\}$, except that there is an extra coefficient $2$ in the second and third lines. This extra factor is due to the difference in initialization schemes--in our paper the entries of hidden layer matrices are randomly generated with variance $2/m$, while in \citet{jacot2018neural} the variance of the random initialization is $1/m$. We remark that this extra factor $2$ in Definition \ref{def:nueraltangentkernel} will remove the exponential dependence on the network depth $L$ in the kernel matrix, which is appealing. In fact, it is easy to check that under our scaling, the diagonal entries of $\bSigma^{(L)}$ are all $1$'s, and the diagonal entries of $\tilde\bTheta^{(L)}$ are all $L$'s.




The following lemma is a summary of Theorem~1 and Proposition~2 in \citet{jacot2018neural}, which ensures that $\bTheta^{(L)}$ is the infinite-width limit of the Gram matrix $( m^{-1} \la \nabla_{\Wb}f_{\Wb^{(1)}}(\xb_i) , \nabla_{\Wb}f_{\Wb^{(1)}}(\xb_j) \ra )_{n\times n} $, and is positive-definite as long as no two training inputs are parallel. 
\begin{lemma}[\citet{jacot2018neural}]\label{lemma:neuraltangentkernelconvergence}
For an $L$ layer ReLU network with parameter set $\Wb^{(1)}$ initialized in Algorithm~\ref{alg:SGDrandominit}, 
as the network width $m\rightarrow \infty$\footnote{The original result by \citet{jacot2018neural} requires that the widths of different layers go to infinity sequentially. Their result was later improved by \citet{yang2019scaling} such that the widths of different layers can go to infinity simultaneously.}, it holds that
\begin{align*}
    m^{-1} \la \nabla_{\Wb}f_{\Wb^{(1)}}(\xb_i) , \nabla_{\Wb}f_{\Wb^{(1)}}(\xb_j) \ra \xrightarrow{\PP} \bTheta_{i,j}^{(L)},
\end{align*}
where the expectation is taken over the randomness of $\Wb^{(1)}$. 
Moreover, as long as each pair of inputs among $\xb_1,\ldots,\xb_n\in S^{d-1}$ are not parallel, $ \bTheta^{(L)} $ is positive-definite.
\end{lemma}

\begin{remark}
Lemmas~\ref{lemma:neuraltangentkernelconvergence} clearly shows the difference between our neural tangent kernel matrix $\bTheta^{(L)}$ in Definition~\ref{def:nueraltangentkernel} and the Gram matrix $\Kb^{(L)}$ defined in Definition~5.1 
in \citet{du2018gradientdeep}. 
For any $i,j\in [n]$, by Lemma~\ref{lemma:neuraltangentkernelconvergence} we have 
\begin{align*}
    \bTheta^{(L)}_{i,j} = \lim_{m\rightarrow \infty} m^{-1} \textstyle{\sum_{l=1}^L}  \la \nabla_{\Wb_l}f_{\Wb^{(1)}}(\xb_i) , \nabla_{\Wb_l}f_{\Wb^{(1)}}(\xb_j) \ra .
\end{align*}
In contrast, the corresponding entry in $\Kb^{(L)}$ is
\begin{align*}
    \Kb^{(L)}_{i,j} = \lim_{m\rightarrow \infty} m^{-1} \la \nabla_{\Wb_{L-1}}f_{\Wb^{(1)}}(\xb_i) , \nabla_{\Wb_{L-1}}f_{\Wb^{(1)}}(\xb_j) \ra .
\end{align*}
It can be seen that our definition of kernel matrix takes all layers into consideration, while \citet{du2018gradientdeep}  only considered the last hidden layer (i.e., second last layer). Moreover, it is clear that $\bTheta^{(L)} \succeq \Kb^{(L)}$. Since the smallest eigenvalue of the kernel matrix plays a key role in the analysis of optimization and generalization of over-parameterized neural networks \citep{du2018gradient,du2018gradientdeep,arora2019fine}, our neural tangent kernel matrix can potentially lead to better bounds than the Gram matrix studied in \citet{du2018gradientdeep}. 
\end{remark}

\begin{corollary}\label{col:expectederrorbound_kernel}
    Let $\yb = (y_1,\ldots,y_n)^\top $ and $\lambda_0 = \lambda_{\min}(\bTheta^{(L)})$.
    For any $\delta \in( 0,e^{-1} ]$, there exists $ \tilde{m}^*(\delta,L,n,\lambda_0)$ that only depends on $\delta,L,n$ and $\lambda_0$ such that if $m \geq \tilde{m}^*(\delta,L,n,\lambda_0)$, then with probability at least $1 - \delta$ over the randomness of $\Wb^{(1)}$, 
    the output of Algorithm~\ref{alg:SGDrandominit} with step size $\eta = \kappa\cdot \inf_{\tilde y_i y_i \geq 1} \sqrt{\tilde\yb^\top (\bTheta^{(L)})^{-1} \tilde\yb } / (m\sqrt{n}) $ 
    for some small enough absolute constant $\kappa$ satisfies
    \begin{align*}
        \EE \big[ L_{\cD}^{0-1}( \hat\Wb ) \big] \leq \tilde \cO\Bigg[ L\cdot \inf_{\tilde y_i y_i \geq 1} \sqrt{\frac{ \tilde\yb^\top (\bTheta^{(L)})^{-1} \tilde\yb }{n}} \Bigg] + \cO\Bigg[ \sqrt{\frac{\log(1 / \delta)}{n}} \Bigg],
    \end{align*}
    where the expectation is taken over the uniform draw of $\hat\Wb$ from $\{\Wb^{(1)},\ldots, \Wb^{(n)}\}$. 
\end{corollary}
Apparently, by choosing $\tilde\yb = \yb$ in Corollary~\ref{col:expectederrorbound_kernel}, one can also obtain a bound on the expected $0-1$ error of the form $ \tilde \cO \big[ L\cdot \sqrt{ \yb^\top (\bTheta^{(L)})^{-1} \yb /n} \big] + \cO\big[ \sqrt{\log(1 / \delta)/n} \big]$. 

\begin{remark}
Corollary~\ref{col:expectederrorbound_kernel} gives an algorithm-dependent generalization error bound of over-parameterized $L$-layer neural networks trained with SGD. It is worth noting that recently \citet{arora2019fine} gives a generalization bound $\tilde\cO\big(  \sqrt{  \yb^\top (\Hb^{\infty})^{-1} \yb / n  }\big)$ for two-layer networks with fixed second layer weights, 
where $\Hb^{\infty}$ is defined as
\begin{align*}
    \Hb^{\infty}_{i,j} = \la \xb_i,\xb_j \ra\cdot \EE_{\wb\sim N(\mathbf{0}, \Ib)}[ \sigma'( \wb^\top \xb_i ) \sigma'(\wb^\top \xb_j )]. 
\end{align*}
Our result in Corollary~\ref{col:expectederrorbound_kernel} can be specialized to two-layer neural networks by choosing $L=2$, and yields a bound $\tilde\cO\big( \sqrt{  \yb^\top (\bTheta^{(2)})^{-1} \yb / n  }\big)$,
where
\begin{align*}
    \bTheta^{(2)}_{i,j} = \Hb^{\infty}_{i,j}  +  2\cdot \EE_{\wb\sim N(\mathbf{0}, \Ib)}[ \sigma( \wb^\top \xb_i ) \sigma(\wb^\top \xb_j )].
\end{align*}
Here the extra term $2\cdot \EE_{\wb\sim N(\mathbf{0}, \Ib)}[ \sigma( \wb^\top \xb_i ) \sigma(\wb^\top \xb_j )]$ corresponds to the training of the second layer--it is the limit of $\frac{1}{m} \la \nabla_{\Wb_2} f_{\Wb^{(1)}}(\xb_i) , \nabla_{\Wb_2} f_{\Wb^{(1)}}(\xb_j) \ra $.
Since we have $\bTheta^{(2)} \succeq \Hb^{\infty}$, 
our bound is sharper than theirs.
This comparison also shows that, our result generalizes the result in \citet{arora2019fine} from two-layer, fixed second layer networks to deep networks with all parameters being trained.  
\end{remark}

\begin{remark}
Corollary~\ref{col:expectederrorbound_kernel} is based on the
asymptotic convergence result in Lemma~\ref{lemma:neuraltangentkernelconvergence}, which does not show how wide the network need to be in order to make the Gram matrix close enough to the NTK matrix.
Very recently,  \citet{arora2019exact} provided a non-asymptotic convergence result for the Gram matrix, and showed the equivalence between an infinitely wide network trained by gradient flow and a kernel regression predictor using neural tangent kernel, which suggests that the generalization of deep neural networks trained by gradient flow can potentially be measured by the corresponding NTK. Utilizing this non-asymptotic convergence result, one can potentially specify the detailed dependency of $\tilde{m}^*(\delta,L,n,\lambda_0)$ on $\delta$, $L$, $n$ and $\lambda_0$ in Corollary~\ref{col:expectederrorbound_kernel}.
\end{remark}

\begin{remark}
Corollary~\ref{col:expectederrorbound_kernel} demonstrates that the generalization bound given by Theorem~\ref{thm:expectederrorbound} does not increase with network width $m$, as long as $m$ is large enough. Moreover, it provides a clear characterization of the classifiability of data. In fact, the $\sqrt{ \tilde\yb^\top (\bTheta^{(L)})^{-1} \tilde\yb }$ factor in the generalization bound given in Corollary~\ref{col:expectederrorbound_kernel} is exactly the NTK-induced RKHS norm of the kernel regression classifier on data $\{(\xb_i,\tilde y_i)\}_{i=1}^n$. Therefore, if $y = f^*(\xb)$ for some $f^*(\cdot)$ with bounded norm in the NTK-induced reproducing kernel Hilbert space (RKHS), then over-parameterized neural networks trained with SGD generalize well. In Appendix~\ref{section:experimental_results}, we provide some numerical evaluation of the leading terms in the generalization bounds in Theorem~\ref{thm:expectederrorbound} and Corollary~\ref{col:expectederrorbound_kernel} to demonstrate that they are very informative on real-world datasets.
\end{remark}

\section{Proof of Main Theory}\label{section:proof_in_mainpaper}
In this section we provide the proof of Theorem~\ref{thm:expectederrorbound} and Corollary~\ref{col:expectederrorbound_kernel}, and explain the intuition behind the proof. 
For notational simplicity, for $i\in[n]$ we denote $L_i(\Wb) = L_{(\xb_i,y_i)} (\Wb)$. 

\subsection{Proof of Theorem~\ref{thm:expectederrorbound}}

Before giving the proof of Theorem~\ref{thm:expectederrorbound}, we first introduce several lemmas. 
The following lemma states that near initialization, the neural network function is almost linear in terms of its weights. 

\begin{lemma}\label{lemma:semilinear}
There exists an absolute constant $\kappa$ such that, with probability at least $1 - \cO(nL^2) \cdot \exp[-\Omega(m\omega^{2/3}L )] $ over the randomness of $\Wb^{(1)}$, for all $i\in [n]$ and $\Wb,\Wb'\in \cB(\Wb^{(1)},\omega)$ with $ \omega \leq \kappa L^{-6} [\log(m)]^{-3/2}$, it holds uniformly that 
\begin{align*}
    | f_{\Wb'}(\xb_i) - f_{\Wb}(\xb_i) - \la \nabla f_{\Wb}(\xb_i) , \Wb' - \Wb \ra | \leq \cO\Big( \omega^{1/3}L^2\sqrt{m\log(m)} \Big) \cdot \textstyle{\sum_{l=1}^{L-1}} \| \Wb_l' - \Wb_l \|_2.
\end{align*}
\end{lemma}

Since the cross-entropy loss $\ell(\cdot)$ is convex, given Lemma~\ref{lemma:semilinear}, we can show in the following lemma that near initialization,  $L_i(\Wb)$ is also almost a convex function of $\Wb$ for any $i\in [n]$.



\begin{lemma}\label{lemma:semiconvex}
There exists an absolute constant $\kappa$ such that, with probability at least $1 - \cO(nL^2)\cdot \exp[-\Omega(m\omega^{2/3} L)]$ over the randomness of $\Wb^{(1)}$, for any $\epsilon > 0$, $i\in[n]$ and $\Wb,\Wb'\in \cB(\Wb^{(1)},\omega)$ with $\omega \leq \kappa L^{-6} m^{-3/8} [\log(m)]^{-3/2} \epsilon^{3/4}$,
it holds uniformly that
\begin{align*}
    L_i(\Wb') \geq L_i(\Wb) + \la \nabla_{\Wb} L_i(\Wb) , \Wb' - \Wb \ra - \epsilon. 
\end{align*}
\end{lemma}

The locally almost convex property of the loss function given by Lemma~\ref{lemma:semiconvex} implies that the dynamics of Algorithm~\ref{alg:SGDrandominit} is similar to the dynamics of convex optimization. We can therefore derive a bound of the cumulative loss. The result is given in the following lemma. 
\begin{lemma}\label{lemma:convergence_SGD}
    For any $\epsilon, \delta, R > 0$, 
    there exists 
    \begin{align*}
        m^* (\epsilon, \delta, R,L) = \tilde\cO\big( \poly(R,L) \big)\cdot \epsilon^{-14} 
        \cdot \log(1 / \delta)
    \end{align*}
    such that if $m \geq m^* (\epsilon, \delta, R,L)$, 
    then with probability at least $1 - \delta$ over the randomness of $\Wb^{(1)}$, for any $\Wb^* \in \cB(\Wb^{(1)}, R m^{-1/2})$, 
    Algorithm~\ref{alg:SGDrandominit} with 
    $\eta = \nu \epsilon / (Lm)$, $n = L^2 R^2 /(2 \nu \epsilon^2)$ 
    for some small enough absolute constant $\nu$ has the following cumulative loss bound:
    \begin{align*}
        \textstyle{\sum_{i=1}^n} L_i(\Wb^{(i)}) &\leq \textstyle{\sum_{i=1}^n} L_i(\Wb^{*}) + 3n\epsilon.
    \end{align*}
\end{lemma}

We now finalize the proof by applying an online-to-batch conversion argument \citep{cesa2004generalization}, and use Lemma~\ref{lemma:semilinear} to relate the neural network function with a function in the NTRF function class. 

\begin{proof}[Proof of Theorem~\ref{thm:expectederrorbound}]
For $i\in[n]$, let $L^{0-1}_i(\Wb^{(i)}) = \ind\big\{ y_i \cdot f_{\Wb^{(i)}}(\xb_i) < 0\big\}$. 
Since cross-entropy loss satisfies
$\ind\{ z \leq 0\} \leq 4\ell(z)$, we have 
$ L^{0-1}_i(\Wb^{(i)}) \leq 4L_i(\Wb^{(i)})$. Therefore, setting $\epsilon  = LR/\sqrt{2\nu n}$ in Lemma~\ref{lemma:convergence_SGD} gives that, if $\eta$ is set as $ \sqrt{\nu/2}R/(m\sqrt{n})$, then with probability at least $1 - \delta$,
\begin{align}\label{eq:expectederrorbound_eq1}
    \frac{1}{n}\sum_{i=1}^n  L^{0-1}_i(\Wb^{(i)}) &\leq \frac{4}{n}\sum_{i=1}^n L_i(\Wb^*) + \frac{12}{\sqrt{2\nu }} \cdot \frac{LR}{\sqrt{n}}.
\end{align}
Note that for any $i\in [n]$, $\Wb^{(i)}$ only depends on $(\xb_1,y_1),\ldots, (\xb_{i-1},y_{i-1})$ and is independent of $(\xb_i,y_i)$. 
Therefore by Proposition~1 in \citet{cesa2004generalization}, with probability at least $1 - \delta$ we have
\begin{align}\label{eq:expectederrorbound_eq2}
    \frac{1}{n} \sum_{i=1}^n  L_{\cD}^{0-1}(\Wb^{(i)}) \leq \frac{1}{n}\sum_{i=1}^n  L^{0-1}_i(\Wb^{(i)}) + \sqrt{\frac{2\log( 1/ \delta )}{n}}.
\end{align}
By definition, we have
$\frac{1}{n} \sum_{i=1}^n  L_{\cD}^{0-1}(\Wb^{(i)}) = \EE \big[ L_{\cD}^{0-1}( \hat\Wb ) \big]$. Therefore
combining \eqref{eq:expectederrorbound_eq1} and \eqref{eq:expectederrorbound_eq2} and applying union bound, we obtain that with probability at least $1 - 2\delta$, 
\begin{align}\label{eq:expectederrorbound_eq3}
    \EE \big[ L_{\cD}^{0-1}( \hat\Wb ) \big] \leq \frac{4}{n}\sum_{i=1}^n L_i(\Wb^*) + \frac{12}{\sqrt{2\nu }} \cdot \frac{LR}{\sqrt{n}} + \sqrt{\frac{2\log(1 / \delta)}{n}}
\end{align}
for all $\Wb^* \in \cB(\Wb^{(1)}, R m^{-1/2})$. 
We now compare the neural network function $f_{\Wb^*}(\xb_i)$ with the function $F_{\Wb^{(1)},\Wb^*}(\xb_i) := f_{\Wb^{(1)}}(\xb_i) + \la \nabla f_{\Wb^{(1)}}(\xb_i) , \Wb^* - \Wb^{(1)} \ra \in \cF(\Wb^{(1)}, R)$. 
We have
\begin{align*}
    L_i(\Wb^{*}) 
    &\leq \ell[y_i\cdot F_{\Wb^{(1)},\Wb^*}(\xb_i) ] + \cO\Big( (Rm^{-1/2})^{1/3}L^2\sqrt{m\log(m)} \Big) \cdot \textstyle{\sum_{l=1}^{L-1}} \big\| \Wb_l^* - \Wb_l^{(1)} \big\|_2\\
    &\leq \ell[y_i\cdot F_{\Wb^{(1)},\Wb^*}(\xb_i) ] +  \cO\Big( L^3\sqrt{m\log(m)} \Big) \cdot R^{4/3} \cdot m^{-2/3} \\
    &\leq \ell[y_i\cdot F_{\Wb^{(1)},\Wb^*}(\xb_i) ] + LR n^{-1/2},
\end{align*}
where the first inequality is by the $1$-Lipschitz continuity of $\ell(\cdot)$ and  Lemma~\ref{lemma:semilinear}, the second inequality is by $\Wb^*\in \cB(\Wb^{(1)},R m^{-1/2})$, and last inequality holds as long as $m \geq C_1 R^{2} L^{12} [\log(m)]^{3} n^3 $ for some large enough absolute constant $C_1$. 
Plugging the inequality above into \eqref{eq:expectederrorbound_eq3} gives
\begin{align*}
    \EE \big[ L_{\cD}^{0-1}( \hat\Wb ) \big] \leq \frac{4}{n}\sum_{i=1}^n \ell[y_i\cdot F_{\Wb^{(1)},\Wb^*}(\xb_i) ] + \bigg(1+\frac{12}{\sqrt{2\nu }}\bigg) \cdot \frac{LR}{\sqrt{n}} + \sqrt{\frac{2\log(1 / \delta)}{n}}.
\end{align*}
Taking infimum over $\Wb^* \in \cB(\Wb^{(1)}, R m^{-1/2})$ and rescaling $\delta$ finishes the proof.
\end{proof}

\subsection{Proof of Corollary~\ref{col:expectederrorbound_kernel}}
In this subsection we prove Corollary~\ref{col:expectederrorbound_kernel}. The following lemma shows that at initialization, with high probability, the neural network function value at all the training inputs are of order $\tilde\cO(1)$. 
\begin{lemma}\label{lemma:initialfunctionvaluebound}
For any $\delta > 0$, if $m\geq K L\log(nL/\delta)$ for a large enough absolute constant $K$, then with probability at least $1 - \delta$,  
$ |f_{\Wb^{(1)}} (\bx_i)| \leq \cO(\sqrt{\log( n / \delta)}) $
for all $i\in[n]$.
\end{lemma}

We now present the proof of Corollary~\ref{col:expectederrorbound_kernel}. The idea is to construct suitable target values $ \hat y_1,\ldots, \hat y_n$, and then bound the norm of the solution of the linear equations $ \hat y_i = \la \nabla f_{\Wb^{(1)}}(\xb_i), \Wb \ra $, $i\in[n]$. In specific, for any $\tilde \yb$ with $\tilde y_i y_i \geq 1$, we examine the \textit{minimum distance solution} to $\Wb^{(1)}$ that fit the data $\{(\xb_i,\tilde y_i)\}_{i=1}^n$ well and use it to construct a specific function in $ \cF\big(\Wb^{(1)}, \tilde\cO\big( \sqrt{\tilde\yb^\top (\bTheta^{(L)})^{-1} \tilde\yb } \big) \big)$.

\begin{proof}[Proof of Corollary~\ref{col:expectederrorbound_kernel}]
Set $B = \log \{ 1/[\exp(n^{-1/2}) - 1 ]\} = \cO(\log(n))$, then for cross-entropy loss we have $\ell(z) \leq n^{-1/2}$ for $z \geq B $. Moreover, let $B' = \max_{i\in[n]} |f_{\Wb^{(1)}} (\bx_i)| $. Then by Lemma~\ref{lemma:initialfunctionvaluebound}, with probability at least $1 - \delta$, $B' \leq \cO(\sqrt{\log( n / \delta)})$ for all $i\in [n]$. For any $\tilde y$ with $\tilde y_i y_i \geq 1$, let $\overline{B} = B + B'$ and $\hat\yb = \overline{B}\cdot \tilde \yb$, 
then it holds that for any $i\in[n]$, 
\begin{align*}
    y_i\cdot [\hat y_i + f_{\Wb^{(1)}} (\bx_i) ] = y_i\cdot \hat y_i + y_i\cdot f_{\Wb^{(1)}} (\bx_i) \geq B + B' - B' \geq B,
\end{align*}
and therefore
\begin{align}\label{eq:learnabledata_eq1}
    \ell\{ y_i\cdot [\hat y_i + f_{\Wb^{(1)}} (\bx_i) ]\} \leq n^{-1/2},~i\in [n].
\end{align}
Denote 
$\Fb = m^{-1/2}\cdot (\mathrm{vec}[ \nabla f_{\Wb^{(1)}}(\xb_1) ], \ldots, \mathrm{vec}[ \nabla f_{\Wb^{(1)}}(\xb_n) ] ) \in \RR^{ [md + m + m^2(L-2) ] \times n }$. 
Note that entries of $\bTheta^{(L)}$ are all bounded by $L$. Therefore, the largest eigenvalue of $\bTheta^{(L)}$ is at most $nL$, and we have $\tilde\yb^\top (\bTheta^{(L)})^{-1} \tilde\yb \geq n^{-1}L^{-1} \| \tilde\yb \|_2^2  = L^{-1}$. 
By Lemma~\ref{lemma:neuraltangentkernelconvergence} and standard matrix perturbation bound, there exists $m^*(\delta,L,n,\lambda_0)$ such that, if $m \geq m^*(\delta,L,n,\lambda_0)$, then with probability at least $1 - \delta$, $ \Fb^\top \Fb$ is strictly positive-definite and
\begin{align}\label{eq:learnabledata_eq2}
    \| (\Fb^\top \Fb)^{-1} - (\bTheta^{(L)})^{-1} \|_2  \leq \inf_{\tilde y_i y_i \geq 1} \tilde\yb^\top (\bTheta^{(L)})^{-1} \tilde\yb / n.
\end{align}
Let $\Fb = \Pb \bLambda \Qb^\top$ be the singular value decomposition of $\Fb$, 
where $\Pb\in \RR^{m \times n},\Qb\in \RR^{n \times n} $ have orthogonal columns, and $\bLambda \in \RR^{n\times n}$ is a diagonal matrix. Let $\wb_{\mathrm{vec}} = \Pb \bLambda^{-1} \Qb^\top \hat\yb $, then we have
\begin{align}\label{eq:learnabledata_eq3}
    \Fb^\top \wb_{\mathrm{vec}} = (\Qb \bLambda \Pb^\top) ( \Pb \bLambda^{-1} \Qb^\top \hat\yb) = \hat\yb.
\end{align}
Moreover, by direct calculation we have
\begin{align*}
    \| \wb_{\mathrm{vec}} \|_2^2 = \| \Pb \bLambda^{-1} \Qb^\top \hat\yb \|_2^2 = \| \bLambda^{-1} \Qb^\top \hat\yb \|_2^2 = \hat\yb^\top \Qb \bLambda^{-2} \Qb^\top \hat\yb = \hat\yb^\top (\Fb^\top \Fb)^{-1} \hat\yb.
\end{align*}
Therefore by \eqref{eq:learnabledata_eq2} and the fact that $\|\hat\yb\|_2^2 = \overline{B}^2 n$, we have
\begin{align*}
    \| \wb_{\mathrm{vec}} \|_2^2 &= \hat\yb^\top [ (\Fb^\top \Fb)^{-1} - (\bTheta^{(L)})^{-1} ] \hat\yb + \hat\yb^\top (\bTheta^{(L)})^{-1} \hat\yb \\
    & \leq \overline{B}^2 \cdot n \cdot \| (\Fb^\top \Fb)^{-1} - (\bTheta^{(L)})^{-1} \|_2 + \overline{B}^2\cdot \tilde\yb^\top (\bTheta^{(L)})^{-1} \tilde\yb\\
    &\leq 2\overline{B}^2\cdot \tilde\yb^\top (\bTheta^{(L)})^{-1} \tilde\yb.
\end{align*}
Let $\Wb\in \cW$ be the parameter collection reshaped from $ m^{-1/2} \wb_{\mathrm{vec}}$. Then clearly 
\begin{align*}
    \| \Wb_l \|_F \leq m^{-1/2} \| \wb_{\mathrm{vec}} \|_2 \leq \tilde\cO\Big( \sqrt{\tilde\yb^\top (\bTheta^{(L)})^{-1} \tilde\yb } \cdot m^{-1/2} \Big),
\end{align*}
and therefore $\Wb \in \cB\big(\mathbf{0} , \cO\big( \sqrt{\tilde\yb^\top (\bTheta^{(L)})^{-1} \tilde\yb } \cdot m^{-1/2} \big) \big)$.  Moreover, by \eqref{eq:learnabledata_eq3}, we have $\hat y_i =  \la \nabla_{\Wb} f_{\Wb^{(1)}}(\xb_i) , \Wb \ra$. 
Plugging this into \eqref{eq:learnabledata_eq1} then gives
\begin{align*}
    \ell\big\{ y_i\cdot \big[ f_{\Wb^{(1)}}(\xb_i) +  \la \nabla_{\Wb} f_{\Wb^{(1)}}(\bx_i) , \Wb \ra  \big]\big\} \leq n^{-1/2}. 
\end{align*}
Since $\hat f(\cdot ) = f_{\Wb^{(1)}}(\cdot) +  \la \nabla_{\Wb} f_{\Wb^{(1)}}(\cdot) , \Wb \ra \in \cF\big(\Wb^{(1)}, \tilde\cO\big( \sqrt{\tilde\yb^\top (\bTheta^{(L)})^{-1} \tilde\yb } \big) \big)$, applying Theorem~\ref{thm:expectederrorbound} and taking infimum over $\tilde \yb$  
completes the proof.
\end{proof}

\section{Conclusions and Future Work}
In this paper we provide an expected $0$-$1$ error bound for wide and deep ReLU networks trained with SGD. This generalization error bound is measured by the NTRF function class. The connection to the neural tangent kernel function studied in \citet{jacot2018neural} is also discussed. 
Our result covers a series of recent  generalization bounds for wide enough neural networks, and provides better bounds.

An important future work is to improve the over-parameterization condition in Theorem~\ref{thm:expectederrorbound} and Corollary~\ref{col:expectederrorbound_kernel}. Other future directions include proving sample complexity lower bounds in the over-parameterized regime, implementing the results in \citet{jain2019making} to obtain last iterate bound of SGD, and establishing uniform convergence based generalization bounds for over-parameterized neural networks with methods developped in \citet{bartlett2017spectrally,neyshabur2017pac,long2019size}. 



\section*{Acknowledgement}
We would like to thank Peter Bartlett for a valuable discussion, and Simon S. Du for pointing out a related work \citep{arora2019exact}. We also thank the anonymous reviewers and area chair for their helpful comments. This research was sponsored in part by the National Science Foundation CAREER Award IIS-1906169, IIS-1903202, and Salesforce Deep Learning Research Award. The views and conclusions contained in this paper are those of the authors and should not be interpreted as representing any funding agencies.



\appendix

\section{Comparison with Recent Results}\label{section:comparisontorecent}
In this section we compare our result in Theorem~\ref{thm:expectederrorbound} with recent generalization error bounds for over-paramerized neural networks by \citet{cao2019generalization,yehudai2019power,e2019comparative}, and backup our discussions in Remark~\ref{remark:comparisiontocao} and Remark~\ref{remark:comparisiontoyehudai&e}. 

\subsection{Comparison with \citet{cao2019generalization}}\label{section:comparisontocao}
In this section we provide direct comparison between our result in Theorem~\ref{thm:expectederrorbound} and Theorem~4.4 in \citet{cao2019generalization}. To concretely compare these two results, we apply our result to the setting studied in \citet{cao2019generalization}, which is based on the following assumption.

\begin{assumption}\label{assump:nonlinearseparable}
There exist a constant $\gamma > 0$ and
\begin{align*}
f(\cdot)\in \bigg\{ f(\xb) = \int_{\RR^d} c(\ub)\sigma(\ub^\top\xb) p(\ub) \mathrm{d}\ub : \| c(\cdot) \|_{\infty} \leq 1 \bigg\},
\end{align*}
where $p(\ub)$ the density of standard Gaussian vectors, such that $ y\cdot f(\xb) \geq \gamma $ for all $(\xb,y) \in \supp(\cD)$.
\end{assumption}
Under Assumption~\ref{assump:normalizeddata} and Assumption~\ref{assump:nonlinearseparable}, in order to train the network to achieve $\epsilon$ expected $0$-$1$ loss,  \citet{cao2019generalization} gave a sample complexity of order $\tilde\cO( \poly(2^L,\gamma^{-1}) \cdot \epsilon^{-4} )$. In comparison, our result in Theorem~\ref{thm:expectederrorbound} leads to the following corollary.

\begin{corollary}\label{col:comparisonwithcao}
Under Assumption~\ref{assump:normalizeddata} and Assumption~\ref{assump:nonlinearseparable},
for any $\delta \in (0,e^{-1}]$, there exists 
    \begin{align*}
        m^* (\delta, \gamma, L,n) = \tilde\cO\big( \poly(2^L,\gamma^{-1}) \big)\cdot  n^7 \cdot \log(1 / \delta)
    \end{align*}
    such that if $m \geq m^* (\delta, R,L,n)$, then with probability at least $1 - \delta$ over the randomness of $\Wb^{(1)}$, the parameters given by Algorithm~\ref{alg:SGDrandominit} with $\eta = \kappa\cdot R / (m\sqrt{n}) $ 
    for some small enough absolute constant $\kappa$ satisfies
    \begin{align*}
        \EE \big[ L_{\cD}^{0-1}( \hat\Wb ) \big] \leq \tilde\cO\Bigg( \frac{2^L\cdot \gamma^{-1}}{\sqrt{n}} \Bigg),
    \end{align*}
    where the expectation is taken over the draws of training examples $\{(\xb_i, y_i)\}_{i=1}^n$ as well as the uniform draw of $\hat\Wb$ from $\{\Wb^{(1)},\ldots, \Wb^{(n)}\}$.
\end{corollary}

By setting the expected $0$-$1$ loss bound to $\epsilon$, we obtain a sample complexity of order $\tilde\cO(4^L\cdot \gamma^{-2} \epsilon^{-2})$, which is better than the sample complexity given in \citet{cao2019generalization} by a factor of $\epsilon^{-2}$.


\subsection{Comparison with \citet{yehudai2019power,e2019comparative}}\label{section:comparisontoyehudai&e}
Here we give a detailed explanation to Remark~\ref{remark:comparisiontoyehudai&e}, where we compare our result with  \citet{yehudai2019power,e2019comparative}. 
The reference function classes studied in these two papers share the same general form: 
\begin{align*}
    \big\{ f(x) = \Wb_2 \sigma(\Wb_1^{(1)} \xb): \|\Wb_2\|_F \leq C m^{-1/2}  \big\},
\end{align*}
where $C$ is a constant, and $\Wb_1^{(1)}\in \RR^{m\times d}$ is the first layer parameter matrix whose rows are sampled from certain distribution $\pi$ associated to the initialization scheme. Specifically, \citet{yehudai2019power} studied the case where $\pi$ is the uniform distribution over the $d$-dimensional cube $[-d^{-1/2}, d^{-1/2}]^d$, while \citet{e2019comparative} studied the uniform distribution over the sphere $S^{d-1}$. By standard concentration inequality, we can see that in both papers, with high probability, the distribution $\pi$ gives $\Wb_1^{(1)}$ with $\| \Wb_1^{(1)} \|_2 \approx \cO( m^{1/2})$. 
In terms of second layer initialization $\Wb_2^{(1)}$, the generalization results in both papers require that $ \| \Wb_2^{(1)} \|_2 \leq \cO(m^{-1/2}) $. With such a scaling, we can apply the following lemma. 


\begin{lemma}\label{lemma:comparisontoyehudai&e}
Suppose that $\Wb^{(1)} = (\Wb_1^{(1)},\Wb_2^{(1)}) \in \RR^{m\times d}\times \RR^{1\times m}$ be weights satisfying $\|\Wb_2^{(1)}\|_F \leq K m^{-1/2}$ for some $K = \tilde\cO(1)$, then
\begin{align*}
    \big\{ f(x) = \Wb_2 \sigma(\Wb_1^{(1)} \xb): \|\Wb_2\|_F \leq C m^{-1/2}  \big\} \subseteq \cF,
\end{align*}
where
\begin{align*}
    \cF = \big\{ \Wb_2^{(1)} \sigma(\Wb_1^{(1)} \xb) + \Wb_2 \sigma(\Wb_1^{(1)} \xb): \|\Wb_2\|_F \leq (C+K)\cdot m^{-1/2} \big\},
\end{align*}
and $\sigma(\cdot)$ is the activation function of interest.
\end{lemma}
We compare our result with the bounds given by \citet{yehudai2019power,e2019comparative} by comparing the reference function classes we use. Apparently, a larger reference function class in general gives a better generalization error bound. Such a comparison requires us to adjust the scaling of initialized parameters. Based on our previous discussion, it is easy to see that the initialized second layer weights in our work and \citet{yehudai2019power,e2019comparative} are all of the same scaling. However, the $\| \cdot \|_2$ of first layer weight matrix in \citet{yehudai2019power,e2019comparative} is larger than ours by a factor of $\sqrt{m}$. Adjusting this scaling difference will give an extra factor $\sqrt{m}$, which matches the $\sqrt{m}$ factor in the definition of our neural network function. 
Note that even after adjusting the scaling of parameters, these random feature function classes are not directly comparable, since the activation functions and the distributions of random weights are different. However, 
an informal comparison can already clearly show the advantage of our result. Moreover, we remark that at least for two-layer networks, our analysis can be easily generalized to other activation functions and initialization methods, and the resulting NTRF class should be strictly larger than the random feature function classes used in \citet{yehudai2019power,e2019comparative}. 
This justifies our discussion in Remark~\ref{remark:comparisiontoyehudai&e}.

\section{Proofs of Technical Lemmas in Section \ref{section:proof_in_mainpaper}}\label{section:proof_main}
In this section we provide the proofs of the technical lemmas in Section \ref{section:proof_in_mainpaper}. 
We first introduce some extra notations. 
Following \citet{allen2018convergence}, for a parameter collection $\Wb$ and $i\in[n]$, we denote
\begin{align*}
    \hb_{i,0} = \xb_i,~ \hb_{i,l} = \sigma(\Wb_l \hb_{i,l-1}), l\in[L-1]
\end{align*}
as the hidden layer outputs of the network. 
We also define binary diagonal matrices 
$$
\Db_{i,l} = \diag\big( \ind\{(\Wb_{l} \hb_{i,l} )_1 > 0 \},\ldots, \ind\{(\Wb_{l} \hb_{i,l} )_m > 0 \} \big), l\in[L-1]. 
$$
For $i\in [n]$ and $l\in [L-1]$, we use $\hb_{i,l}'$, $\Db_{i,l}'$ and $\hb_{i,l}^{(1)}$, $\Db_{i,l}^{(1)}$ to denote the hidden layer outputs and binary diagonal matrices with parameter collections $\Wb'$ and $\Wb^{(1)}$ respectively. We also implement the following matrix product notation which is also used in \citet{zou2018stochastic,cao2019generalization}:
\begin{align*}
    \prod_{r = l_1}^{l_2} \Ab_r :=\left\{
    \begin{array}{ll}
        \Ab_{l_2}\Ab_{l_2-1} \cdots  \Ab_{l_1} & \text{if }l_1\leq l_2 \\
        \Ib & \text{otherwise.}
    \end{array}
    \right.
\end{align*}
With this notation, we have the following matrix product representation of the neural network gradients:
\begin{align*}
    \nabla_{\Wb_{l}} f_{\Wb}(\xb_i)
    =\left\{
    \begin{array}{ll}
        \sqrt{m} \cdot \big[ \hb_{i,l-1} \Wb_{L} \big(\prod_{r=l+1}^{L-1} \Db_{i,r} \Wb_r \big) \Db_{i,l}\big]^\top, & l\in[L-1], \\
        \sqrt{m} \cdot  \hb_{i,L-1}^\top, & l=L.
    \end{array}
    \right.
\end{align*}

\subsection{Proof of Lemma~\ref{lemma:semilinear}}
The following two lemmas are proved based on several results given by \citet{allen2018convergence}. 
Note that in their paper, both the first and the last layers of the network are fixed, which is slightly different from our setting.
We remark that this difference does not affect the result.

\begin{lemma}\label{lemma:normbounds_h}
If $ \omega \leq \cO(L^{-9/2} [\log(m)]^{-3})$, then with probability at least $1 - \cO(nL)\cdot \exp[-\Omega(m\omega^{2/3}L )]$, $1/2 \leq \| \hb_{i,l} \|_2 \leq 3/2$ for all $\Wb \in \cB(\Wb^{(1)}, \omega )$, $i\in[n]$ and $l\in[L-1]$. 
\end{lemma}

\begin{lemma}\label{lemma:normbounds_matproduct}
If $ \omega \leq \cO(L^{-6} [\log(m)]^{-3})$, then with probability at least $1 - \cO(nL^2) \cdot \exp[-\Omega(m\omega^{2/3}L )] $, uniformly over:
\begin{itemize}[leftmargin=*]
    \item any $i\in [n]$, $ 1\leq l_1 < l_2 \leq L-1$
    \item any diagonal matrices $\Db_{i,1}'',\ldots, \Db_{i,L-1}''\in [-1,1]^{m\times m}$ with at most $ \cO(m\omega^{2/3} L)$ non-zero entries, 
\end{itemize}
the following results hold:
\begin{enumerate}[label=(\roman*),leftmargin=*]
    \item For all $\Wb \in \cB(\Wb^{(1)}, \omega )$, $\| \prod_{r = l_1}^{l_2} (\Db_{i,r} + \Db_{i,r}'') \Wb_r \|_2 \leq \cO(\sqrt{L})$. \label{item:normbounds_matproduct_mid}
    \item For all $\Wb \in \cB(\Wb^{(1)}, \omega )$, $\| \Wb_L \prod_{r = l_1}^{L-1} (\Db_{i,r} + \Db_{i,r}'') \Wb_r \|_2 \leq \cO(1)$. \label{item:normbounds_matproduct_last}
    \item For all $\Wb, \Wb' \in \cB(\Wb^{(1)}, \omega )$, \label{item:normbounds_matproduct_last_difference}
    \begin{align*}
    \Bigg\| \Wb_L' \prod_{r = l_1}^{L-1} (\Db_{i,r}' + \Db_{i,r}'') \Wb_r' - \Wb_L \prod_{r = l_1}^{L-1} \Db_{i,r} \Wb_r  \Bigg\|_2 \leq \cO \Big( \omega^{1/3}L^2\sqrt{\log(m)} \Big).
\end{align*}
\end{enumerate}

\end{lemma}

We are now ready to prove Lemma~\ref{lemma:semilinear}.
\begin{proof}[Proof of Lemma~\ref{lemma:semilinear}] 
Since $f_{\Wb'} (\xb_i) = \sqrt{m}\cdot  \Wb_{L}' \hb_{i,L - 1}'$, $f_{\Wb} (\xb_i) = \sqrt{m}\cdot  \Wb_{L} \hb_{i,L - 1}$, by direct calculation, we have
\begin{align*}
    f_{\Wb'}(\xb_i) - F_{\Wb, \Wb'}(\xb_i) &=  - \sqrt{m}\cdot \sum_{l=1}^{L-1} \Wb_L \Bigg(\prod_{r=l+1}^{L-1} \Db_{i,r} \Wb_r \Bigg) \Db_{i,l} (\Wb_l' - \Wb_l) \hb_{i,l-1}\\
    &\quad + \sqrt{m}\cdot \Wb_L' ( \hb_{i,L - 1}' - \hb_{i,L - 1} ).
\end{align*}
By Claim~8.2 in \citet{allen2018convergence}
, there exist diagonal matrices $\Db_{i,l}''\in \RR^{m\times m}$ with entries in $[-1,1]$ such that $\| \Db_{i,l}'' \|_0 \leq \cO(m \omega^{2/3} L)$ and
\begin{align*}
    \hb_{i,L-1} - \hb_{i,L-1}' = \sum_{l=1}^{L-1} \Bigg[ \prod_{r=l+1}^{L-1} (\Db_{i,r}' + \Db_{i,r}'') \Wb_r' \Bigg] (\Db_{i,l}' + \Db_{i,l}'') (\Wb_l - \Wb_l') \hb_{i,l-1}
\end{align*}
for all $i\in [n]$. Therefore
\begin{align*}
    f_{\Wb'}(\xb_i) - F_{\Wb, \Wb'}(\xb_i) &= \sqrt{m}\cdot \sum_{l=1}^{L-1} \Wb_L' \Bigg[ \prod_{r=l+1}^{L-1} (\Db_{i,r}' + \Db_{i,r}'') \Wb_r' \Bigg] (\Db_{i,l}' + \Db_{i,l}'') (\Wb_l - \Wb_l') \hb_{i,l-1} \\
    &\quad - \sqrt{m}\cdot \sum_{l=1}^{L-1} \Wb_L \Bigg(\prod_{r=l+1}^{L-1} \Db_{i,r} \Wb_r \Bigg) \Db_{i,l} (\Wb_l' - \Wb_l) \hb_{i,l-1}.
\end{align*}
By \ref{item:normbounds_matproduct_last_difference} in Lemma~\ref{lemma:normbounds_matproduct}, with probability at least $1 - \cO(nL^2) \cdot \exp[-\Omega(m\omega^{2/3}L )] $, we have
\begin{align*}
    |f_{\Wb'}(\xb_i) - F_{\Wb, \Wb'}(\xb_i)| &\leq \cO\Big( \omega^{1/3}L^2\sqrt{m\log(m)} \Big) \cdot \sum_{l=1}^{L-1} \| \hb_{i.l-1} \|_2 \cdot  \| \Wb_l' - \Wb_l \|_2\\
    &\leq \cO\Big( \omega^{1/3}L^2\sqrt{m\log(m)} \Big) \cdot \sum_{l=1}^{L-1} \| \Wb_l' - \Wb_l \|_2,
\end{align*}
where the last inequality follows by Lemma~\ref{lemma:normbounds_h}. This inequality finishes the proof.
\end{proof}

\subsection{Proof of Lemma~\ref{lemma:semiconvex}}
Intuitively, Lemma~\ref{lemma:semiconvex} follows by the fact that the composition of a convex function and an almost linear function is almost convex. The detailed proof is as follows.
\begin{proof}[Proof of Lemma~\ref{lemma:semiconvex}]
By the convexity of $\ell(z)$, we have
\begin{align*}
    L_i(\Wb') - L_i(\Wb) =  \ell[y_i  f_{\Wb'}(\xb_i) ] - \ell[ y_i f_\Wb(\xb_i) ] \geq  \ell'[ y_i f_\Wb(\xb_i)] \cdot y_i\cdot [ f_{\Wb'}(\xb_i) - f_{\Wb}(\xb_i) ].
\end{align*}
By chain rule, we have 
$$ \sum_{l=1}^L \la \nabla_{\Wb_l} L_i(\Wb), \Wb_l' - \Wb_l \ra = \ell'[ y_i f_\Wb(\xb_i)] \cdot y_i\cdot \la \nabla f_{\Wb}(\xb_i) , \Wb' - \Wb \ra.  $$
Therefore by triangle inequality, we have
\begin{align*}
    \ell'[ y_i f_\Wb(\xb_i)]\cdot y_i\cdot[ f_{\Wb'}(\xb_i) - f_{\Wb}(\xb_i) ] &\geq \ell'[ y_i f_\Wb(\xb_i)] \cdot y_i\cdot \la \nabla f_{\Wb}(\xb_i) , \Wb' - \Wb \ra  - I \\
    &= \textstyle \sum_{l=1}^L \la \nabla_{\Wb_l} L_i(\Wb), \Wb_l' - \Wb_l \ra - I,
\end{align*}
where $I = \big| \ell'[ y_i f_\Wb(\xb_i)] \cdot y_i\cdot \big[ f_{\Wb'}(\xb_i) - f_{\Wb}(\xb_i)  - \la \nabla f_{\Wb}(\xb_i) , \Wb' - \Wb \ra \big] \big|$. Then by upper-bounding $I$ with Lemma~\ref{lemma:semilinear} and the fact that $| \ell'[ y_i f_\Wb(\xb_i)] \cdot y_i | \leq 1$, we have
\begin{align*}
    L_i(\Wb') - L_i(\Wb) &\geq \sum_{l=1}^L \la \nabla_{\Wb_l} L_i(\Wb), \Wb_l' - \Wb_l \ra - \cO\Big( \omega^{1/3}L^2\sqrt{m\log(m)} \Big) \sum_{l=1}^{L-1} \| \Wb_l' - \Wb_l \|_2\\
    &\geq \sum_{l=1}^L \la \nabla_{\Wb_l} L_i(\Wb), \Wb_l' - \Wb_l \ra - \epsilon,
\end{align*}
where the last inequality again follows by $\omega \leq \cO \big( L^{-9/4} m^{-3/8} [\log(m)]^{-3/8} \epsilon^{3/4} \big)$. 
\end{proof}

\subsection{Proof of Lemma~\ref{lemma:convergence_SGD}}

To prove Lemma~\ref{lemma:convergence_SGD}, we first introduce the following lemma which provides an upper bound for the gradient of the neural network function near initialization. 

\begin{lemma}\label{lemma:NNgradient_uppbound}
There exists an absolute constant $\kappa$ such that, with probability at least $1 - \cO(nL^2) \cdot \exp[-\Omega(m\omega^{2/3}L )] $, for all $i\in [n]$, $l\in[L]$ and $\Wb\in \cB(\Wb^{(1)},\omega)$ with $ \omega \leq \kappa L^{-6} [\log(m)]^{-3}$, it holds uniformly that 
\begin{align*}
    \| \nabla_{\Wb_l} f_{\Wb}(\xb_i) \|_F, \|\nabla_{\Wb_l} L_i(\Wb)\|_F \leq \cO(\sqrt{m}).
\end{align*}
\end{lemma}

We now provide the final proof of Lemma~\ref{lemma:convergence_SGD}.

\begin{proof}[Proof of Lemma~\ref{lemma:convergence_SGD}]
Let $\omega = C_1L^{-6} m^{-3/8} [\log(m)]^{-3} \epsilon^{3/4}$, where $C_1$ is a small enough absolute constant such that the conditions on $\omega$ given in Lemmas~\ref{lemma:semiconvex} and \ref{lemma:NNgradient_uppbound} hold. It is easy to see that as long as $m \geq C_1^{-8} R^{8} L^{48} [\log(m)]^{12} \epsilon^{-6} $, we have $\Wb^* \in \cB(\Wb^{(1)}, \omega)$. 
We now show that under our parameter choice, 
$\Wb^{(1)},\ldots,\Wb^{(n)} $ are inside $ \cB(\Wb^{(1)}, \omega)$ as well. 

This result follows by simple induction. Clearly we have $\Wb^{(1)} \in \cB(\Wb^{(1)}, \omega)$. Suppose that $ \Wb^{(1)},\ldots, \Wb^{(i)} \in \cB(\Wb^{(1)}, \omega) $. Then by Lemma~\ref{lemma:NNgradient_uppbound}, for $l\in [L]$ we have $\|\nabla_{\Wb_l} L_i(\Wb^{(i)})\|_F \leq \cO(\sqrt{m})$.
Therefore
\begin{align*}
    \big\| \Wb_l^{(i+1)} - \Wb_l^{(1)} \big\|_F &\leq \sum_{j  = 1}^i\big\| \Wb_l^{(j+1)} - \Wb_l^{(j)} \big\|_F \leq \cO(\sqrt{m} \eta n).
\end{align*}
Plugging in our parameter choice 
$\eta = \nu \epsilon / (Lm)$, $n = L^2 R^2 /(2 \nu \epsilon^2)$ 
for some small enough absolute constant $\nu$ gives
\begin{align*}
     \big\| \Wb_l^{(i+1)} - \Wb_l^{(1)} \big\|_F \leq \cO\big(\sqrt{m}\cdot L R^{2}/(2m\epsilon)\big) \leq \omega,
\end{align*}
where the last inequality holds as long as $m \geq C_2 R^{16} L^{56} [\log(m)]^{12} \epsilon^{-14} $ for some large enough constant $C_2$. Therefore by induction we see that $\Wb^{(1)},\ldots,\Wb^{(n)} \in \cB(\Wb^{(1)}, \omega)$. 
As a result, the conditions of Lemmas~\ref{lemma:semiconvex} and \ref{lemma:NNgradient_uppbound} are satisfied for $\Wb^*$ and $\Wb^{(1)},\ldots, \Wb^{(n)}$.

In the following, we utilize the results of Lemmas~\ref{lemma:semiconvex} and \ref{lemma:NNgradient_uppbound} to prove the bound of cumulative loss. First of all, by Lemma~\ref{lemma:semiconvex}, we have
\begin{align*}
    L_i(\Wb^{(i)}) - L_i(\Wb^{*}) &\leq  \la \nabla_{\Wb} L_i(\Wb^{(i)}), \Wb^{(i)} - \Wb^* \ra + \epsilon \\
    &= \sum_{l=1}^L \frac{\la \Wb_l^{(i)} - \Wb_l^{(i+1)}, \Wb_l^{(i)} - \Wb_l^* \ra }{ \eta } + \epsilon
\end{align*}
Note that for the matrix inner product we have the equality $2\la \Ab, \Bb \ra = \| \Ab \|_F^2 +  \| \Bb \|_F^2 - \| \Ab - \Bb \|_F^2$. Applying this equality to the right hand side above gives
\begin{align*}
    L_i(\Wb^{(i)}) - L_i(\Wb^{*}) \leq \sum_{l=1}^L \frac{ \| \Wb_l^{(i)} - \Wb_l^{(i+1)} \|_F^2 + \| \Wb_l^{(i)} - \Wb_l^* \|_F^2 - \| \Wb_l^{(i+1)} - \Wb_l^* \|_F^2 } {2\eta} + \epsilon.
\end{align*}
By Lemma~\ref{lemma:NNgradient_uppbound}, for $l\in [L]$ we have $\| \Wb_l^{(i)} - \Wb_l^{(i+1)} \|_F \leq \eta\|\nabla_{\Wb_l} L_i(\Wb^{(i)})\|_F \leq \cO(\eta \sqrt{m})$. 
Therefore
\begin{align*}
    L_i(\Wb^{(i)}) - L_i(\Wb^{*}) \leq \sum_{l=1}^L \frac{ \| \Wb_l^{(i)} - \Wb_l^* \|_F^2 - \| \Wb_l^{(i+1)} - \Wb_l^* \|_F^2 } {2\eta} + \cO(L \eta m) + \epsilon.
\end{align*}
Telescoping over $i = 1,\ldots, n$, we obtain
\begin{align*}
    \sum_{i=1}^n L_i(\Wb^{(i)}) &\leq \sum_{i=1}^n L_i(\Wb^{*}) + \sum_{l=1}^L \frac{ \| \Wb_l^{(1)} - \Wb_l^* \|_F^2 } {2\eta} + \cO(L \eta n m) + n\epsilon\\
    &\leq \sum_{i=1}^n L_i(\Wb^{*}) + \frac{ L R^{2} } {2\eta m} + \cO(L \eta n m) + n\epsilon,
\end{align*}
where in the first inequality we simply remove the term $-\| \Wb_l^{(n+1)} - \Wb_l^* \|_F^2/(2\eta) $ to obtain an upper bound, and the second inequality follows by the assumption that $\Wb^*\in \cB(\Wb^{(1)},Rm^{-1/2})$. 
Plugging in the parameter choice 
$\eta = \nu \epsilon / (Lm)$, $n = L^2 R^2 /(2 \nu \epsilon^2)$ 
for some small enough absolute constant $\nu$  gives
\begin{align*}
    \sum_{i=1}^n L_i(\Wb^{(i)}) &\leq \sum_{i=1}^n L_i(\Wb^{*}) + 3n\epsilon,
\end{align*}
which finishes the proof.
\end{proof}

\subsection{Proof of Lemma~\ref{lemma:initialfunctionvaluebound}}
Here we prove Lemma~\ref{lemma:initialfunctionvaluebound}. The proof essentially follows by standard Gaussian tail bound and a bound on the length of last hidden layer output vector.
\begin{proof}[Proof of Lemma~\ref{lemma:initialfunctionvaluebound}]

By Lemma~4.1 in \citet{allen2018convergence}, with probability at least $1 - \cO(nL)\cdot \exp[ -\Omega(m/L) ] > 1 - \delta / 2$ over the randomness of $\Wb^{(1)}_1,\ldots,\Wb^{(1)}_{L-1}$, $\| \hb_{i,L-1}^{(0)} \|_2 \in [1/2,3/2]$ for all $i\in[n]$. Condition on $\Wb^{(1)}_1,\ldots,\Wb^{(1)}_{L-1}$, $f_{\Wb^{(1)}}(\xb_i) = \sqrt{m}\cdot \Wb^{(1)}_{L} \hb_{i,L-1}$ is a Gaussian random variable with variance $\|\hb_{i,L-1}\|_2^2$. Therefore by standard Gaussian tail bound and union bound, with probability at least $1 - \delta$, $|f_{\Wb^{(1)}} (\bx_i)| \leq \cO(\sqrt{\log( n / \delta)})$ for all $i\in [n]$. 
\end{proof}

\section{Proofs of Results in Section~\ref{section:comparisontorecent}}
In this section we provide the proofs of Corollary~\ref{col:comparisonwithcao} and Lemma~\ref{lemma:comparisontoyehudai&e}.

\subsection{Proof of Corollary~\ref{col:comparisonwithcao}}

The following lemma is a simplified version of Lemma~C.2 in \citet{cao2019generalization}. Since the proof is almost the same as the proof of Lemma~C.2 in \citet{cao2019generalization}, except replacing the $\epsilon$-net argument with a simple union bound over $n$ training examples, we omit the proof detail here.

\begin{lemma}\label{lemma:linearseparable}
For any $\delta > 0$, if $m \geq K\cdot 4^L L^4 \gamma^{-2} \log(nL/\delta)$ for some large enough absolute constant $K$, then with probability at least $1 - \delta$, there exists $\balpha_{L-1} \in \RR^m$ such that $y_i\cdot \la \balpha_{L-1} , \hb_{i,L-1} \ra  \geq 2^{-L} \gamma $ for all $i\in [n]$.
\end{lemma}

\begin{proof}[Proof of Corollary~\ref{col:comparisonwithcao}]
Set $B = \log \{ 1/[\exp(n^{-1/2}) - 1 ]\} = \cO(\log(n))$, then for cross-entropy loss we have $\ell(z) \leq n^{-1/2}$ for $z \geq B $. Moreover, let $B' = \max_{i\in[n]} |f_{\Wb^{(1)}} (\bx_i)| $. Then by Lemma~\ref{lemma:initialfunctionvaluebound}, with probability at least $1 - \delta$, $B' \leq \cO(\sqrt{\log( n / \delta)})$ for all $i\in [n]$.

By Lemma~\ref{lemma:linearseparable}, with probability at least $1 - \delta$, there exists $\balpha_{L-1} \in S^{m-1}$ such that $y_i\cdot \la \balpha_{L-1} , \hb_{i,L-1} \ra  \geq 2^{-L} \gamma $ for all $i\in [n]$. Therefore, setting $R = (B + B')\cdot 2^{L} \gamma^{-1} = \tilde\cO(2^{L} \gamma^{-1})$, we have
\begin{align*}
    \Wb = (\mathbf{0},\ldots,\mathbf{0}, Rm^{-1/2}\cdot \balpha_{L-1}^\top) \in \cB(\mathbf{0}, Rm^{-1/2}).
\end{align*}
Moreover, $f^*(\cdot) := f_{\Wb^{(1)}}(\cdot) + \la \nabla_\Wb f_{\Wb^{(1)}}(\cdot), \Wb \ra $ satisfies $f^*\in \cF(\Wb^{(1)},R)$, and
\begin{align*}
    y_i\cdot f^*(\xb_i) &= y_i\cdot f_{\Wb^{(1)}}(\xb_i) + y_i\cdot \la \sqrt{m}\cdot  \hb_{i,L-1}^\top, Rm^{-1/2}\cdot \balpha_{L-1}^\top \ra \\
    &\geq (B + B')\cdot 2^{L} \gamma^{-1} \cdot 2^{-L}\gamma - B'\\
    &\geq B.
\end{align*}
Therefore we have $\ell(y_i\cdot f^*(\xb_i)) \leq \epsilon$, $i\in[n]$. 
Applying Theorem~\ref{thm:expectederrorbound} gives
\begin{align*}
    \EE \big[ L_{\cD}^{0-1}( \hat\Wb ) \big] \leq \tilde\cO\Bigg( \frac{2^L\cdot \gamma^{-1}}{\sqrt{n}} \Bigg) + \cO\Bigg[ \sqrt{\frac{\log(1 / \delta)}{n}} \Bigg] = \tilde\cO\Bigg( \frac{2^L\cdot \gamma^{-1}}{\sqrt{n}} \Bigg).
\end{align*}
This finishes the proof.
\end{proof}

\subsection{Proof of Lemma~\ref{lemma:comparisontoyehudai&e}}
Here we give the proof of  Lemma~\ref{lemma:comparisontoyehudai&e}. It is based on a simple construction.
\begin{proof}[Proof of Lemma~\ref{lemma:comparisontoyehudai&e}] 
For any $f(x)= \Wb_2 \sigma(\Wb_1^{(1)} \xb)$ with $ \|\Wb_2\|_F \leq C m^{-1/2} $, by the assumption that $\|\Wb_2^{(1)}\|_F \leq K m^{-1/2}$ for some $K = \tilde\cO(1)$, we have $\Wb_2' := \Wb_2 - \Wb_2^{(1)}$ satisfies $\|\Wb_2'\|_F\leq (C+K)\cdot m^{-1/2}$. Therefore
\begin{align*}
    f(x)= \Wb_2 \sigma(\Wb_1^{(1)} \xb) = \Wb_2^{(1)} \sigma(\Wb_1^{(1)} \xb) + \Wb_2' \sigma(\Wb_1^{(1)} \xb) \subseteq \cF.
\end{align*}
This finishes the proof.
\end{proof}

\section{Proofs of Lemmas in Section~\ref{section:proof_main}}\label{section:appendixA}
In this section we give the proofs of lemma~\ref{lemma:normbounds_h}, Lemma~\ref{lemma:normbounds_matproduct} and Lemma~\ref{lemma:NNgradient_uppbound} in Section~\ref{section:proof_main}.


\subsection{Proof of Lemma~\ref{lemma:normbounds_h}}

\begin{proof}[Proof of Lemma~\ref{lemma:normbounds_h}]
By Lemma~4.1 in \citet{allen2018convergence}, with probability at least $1 - \cO(nL)\cdot \exp[ -\Omega(m/L) ]$, $\| \hb_{i,l}^{(1)} \|_2 \in [3/4,5/4]$ for all $i\in[n]$ and $l\in [L-1]$. Moreover, by Lemma~5.2 in \citet{allen2018convergence} and the $1$-Lipschitz continuity of $\sigma(\cdot)$, with probability at least $ 1 - \cO(nL)\cdot \exp[-\Omega(m \omega^{2/3}L)]$, $\| \hb_{i,l} - \hb_{i,l}^{(1)} \|_2 \leq \cO(\omega L^{5/2} \sqrt{\log(m)} )$. Therefore by the assumption that $ \omega \leq \cO(L^{-9/2} [\log(m)]^{-3}) $, 
we have $ \| \hb_{i,l} \|_2 \in [1/2, 3/2] $ for all $i\in [n]$ and $l\in[L-1]$.
\end{proof}

\subsection{Proof of Lemma~\ref{lemma:normbounds_matproduct}}
We first introduce the following lemma characterizing the activation changes between networks with two close enough parameter sets $\Wb$ and $\Wb'$. This lemma directly follows by Lemma~8.2 in \citet{allen2018convergence} and triangle inequality.
\begin{lemma}\label{lemma:differencesparsity}
If $\omega \leq \cO(L^{-9/2} [\log(m)]^{-3/2})$, 
then with probability at least $1 - \cO(nL)\cdot \exp[-\Omega(m \omega^{2/3}L)]$, 
\begin{align*}
    \| \Db_{i,l} - \Db_{i,l}' \|_0 \leq \cO(L\omega^{2/3} m)
\end{align*}
for all $\Wb,\Wb' \in \cB( \Wb^{(1)},\omega )$, $i\in [n]$ and $l\in [L-1]$.
\end{lemma}

\begin{proof}[Proof of Lemma~\ref{lemma:normbounds_matproduct}]
We first prove \ref{item:normbounds_matproduct_mid} and \ref{item:normbounds_matproduct_last_difference}, and then use \ref{item:normbounds_matproduct_last_difference} to prove \ref{item:normbounds_matproduct_last}.

By Lemma~\ref{lemma:differencesparsity}, with probability at least $1 - \cO(nL)\cdot \exp(-\Omega(L\omega^{2/3}m))$, $ \| \Db_{i,l} - \Db_{i,l}^{(1)} \|_0 \leq \cO(L\omega^{2/3} m) $ for all $i\in [n]$ and $l\in[L-1]$. Therefore we have $\| \Db_{i,r} + \Db_{i,r}'' - \Db_{i,l}^{(1)} \|_0 \leq \cO(L\omega^{2/3} m) $ for all $i\in [n]$ and $l\in[L-1]$. Therefore by Lemma~5.6 in \citet{allen2018convergence}, with probability at least $1 - \cO(nL^2) \cdot \exp[-\Omega(m\omega^{2/3}L )] $ we have
$ \big\| \prod_{r = l_1}^{l_2} (\Db_{i,r} + \Db_{i,r}'') \Wb_r \big\|_2 \leq \cO(\sqrt{L})$. This completes the proof of \ref{item:normbounds_matproduct_mid} in Lemma~\ref{lemma:normbounds_matproduct}.

Similarly, to prove \ref{item:normbounds_matproduct_last_difference}, applying Lemma~\ref{lemma:differencesparsity} to $\Wb'$ gives that with probability at least $1 - \cO(nL)\cdot \exp(-\Omega(L\omega^{2/3}m))$, $ \| \Db_{i,l}' + \Db_{i,r}'' - \Db_{i,l}^{(1)} \|_0 \leq \cO(L\omega^{2/3} m) $ for all $i\in [n]$ and $l\in[L-1]$. Now by Lemma~5.7 in \citet{allen2018convergence}\footnote{Note that $\sqrt{m}\cdot \Wb_L^{(1)}$ is a random vector following the Gaussian distribution $N(\mathbf{0}, \Ib)$, which matches the distribution of the last layer parameters in \citet{allen2018convergence} for the binary classification case, where the output dimension of the network is $1$.} with $s = \cO(m \omega^{2/3} L)$ to $\Wb$ and $\Wb'$, we have
\begin{align}
    &\sqrt{m}\cdot \Bigg\| \Wb_L^{(1)} \prod_{r = l_1}^{L-1} (\Db_{i,r}' + \Db_{i,r}'') \Wb_r' - \Wb_L^{(1)} \prod_{r = l_1}^{L-1} \Db_{i,r}^{(1)} \Wb_r^{(1)}  \Bigg\|_2 \leq \cO \Big( \omega^{1/3}L^2\sqrt{m\log(m)} \Big),\label{eq:normbounds_matproduct_eq1} \\
    &\sqrt{m}\cdot \Bigg\| \Wb_L^{(1)} \prod_{r = l_1}^{L-1} \Db_{i,r} \Wb_r - \Wb_L^{(1)} \prod_{r = l_1}^{L-1} \Db_{i,r}^{(1)} \Wb_r^{(1)}  \Bigg\|_2 \leq \cO \Big( \omega^{1/3}L^2\sqrt{m\log(m)} \Big). \label{eq:normbounds_matproduct_eq2}
\end{align}
Moreover, by result \ref{item:normbounds_matproduct_mid}, we have
\begin{align}
    &\Bigg\| (\Wb_L'  -  \Wb_L^{(1)}) \prod_{r = l_1}^{L-1} (\Db_{i,r}' + \Db_{i,r}'') \Wb_r' \Bigg\|_2 \leq \cO(\sqrt{L}\omega) \leq \cO \Big( \omega^{1/3}L^2\sqrt{\log(m)} \Big), \label{eq:normbounds_matproduct_eq3}\\
    &\Bigg\| (\Wb_L  -  \Wb_L^{(1)}) \prod_{r = l_1}^{L-1} \Db_{i,r} \Wb_r \Bigg\|_2 \leq \cO(\sqrt{L}\omega) \leq \cO \Big( \omega^{1/3}L^2\sqrt{\log(m)} \Big). \label{eq:normbounds_matproduct_eq4}
\end{align}
Combining equations \eqref{eq:normbounds_matproduct_eq1}, \eqref{eq:normbounds_matproduct_eq2}, \eqref{eq:normbounds_matproduct_eq3}, \eqref{eq:normbounds_matproduct_eq4} and applying triangle inequality gives the desired final result \ref{item:normbounds_matproduct_last_difference}.

Finally to prove \ref{item:normbounds_matproduct_last}, we write
\begin{align*}
    \Bigg\| \Wb_L \prod_{r = l_1}^{L-1} (\Db_{i,r} + \Db_{i,r}'') \Wb_r \Bigg\|_2 & \leq \Bigg\| \Wb_L \prod_{r = l_1}^{L-1} (\Db_{i,r} + \Db_{i,r}'') \Wb_r - \Wb_L^{(1)} \prod_{r = l_1}^{L-1} \Db_{i,r}^{(1)} \Wb_r^{(1)} \Bigg\|_2 \\
    &\quad + \Bigg\| \Wb_L^{(1)} \prod_{r = l_1}^{L-1} \Db_{i,r}^{(1)} \Wb_r^{(1)} \Bigg\|_2.
\end{align*}
Applying \ref{item:normbounds_matproduct_last_difference} and (b) in Lemma~4.4 in \citet{allen2018convergence}, with probability at least $1 - \cO(nL)\cdot \exp[-\Omega(m/L)]$, we obtain
\begin{align*}
    \Bigg\| \Wb_L \prod_{r = l_1}^{L-1} (\Db_{i,r} + \Db_{i,r}'') \Wb_r \Bigg\|_2 \leq \cO \Big( \omega^{1/3}L^2\sqrt{\log(m)} \Big) + \cO(1) = \cO(1).
\end{align*}
This gives \ref{item:normbounds_matproduct_last}.
\end{proof}

\subsection{Proof of Lemma~\ref{lemma:NNgradient_uppbound}}

\begin{proof}[Proof of Lemma~\ref{lemma:NNgradient_uppbound}]
By Lemma~\ref{lemma:normbounds_h}, clearly we have
\begin{align*}
    \| \nabla_{\Wb_l} f_{\Wb} (\xb_i) \|_F = \| \sqrt{m} \cdot \hb_{i,L-1} \|_2 \leq \cO(\sqrt{m})
\end{align*}
for all $\Wb\in \cB(\Wb^{(1)},\omega)$ and $i\in [n]$. 
For $l\in [L - 1]$, by direct calculation we have
\begin{align*}
    \| \nabla_{\Wb_{l}} f_{\Wb}(\xb_i) \|_F &= \sqrt{m} \cdot \Bigg\| \hb_{i,l-1} \Wb_{L} \Bigg(\prod_{r=l+1}^{L-1} \Db_{i,r} \Wb_r \Bigg) \Db_{i,l}\Bigg\|_F\\
    & = \sqrt{m} \cdot \| \hb_{i,l-1} \|_2 \cdot \Bigg\| \Wb_{L} \Bigg(\prod_{r=l+1}^{L-1} \Db_{i,r} \Wb_r \Bigg) \Db_{i,l}\Bigg\|_2.
\end{align*}
Therefore by Lemma~\ref{lemma:normbounds_h} and
\ref{item:normbounds_matproduct_last} in Lemma~\ref{lemma:normbounds_matproduct},
we have
\begin{align*}
    \| \nabla_{\Wb_{l}} f_{\Wb}(\xb_i) \|_F \leq \cO(\sqrt{m}).
\end{align*}
Finally, for $\|\nabla_{\Wb_l} L_i(\Wb^{(i)})\|_F$ we have
\begin{align*}
    \|\nabla_{\Wb_l} L_i(\Wb^{(i)})\|_F \leq \big| \ell'[y_i\cdot f_{\Wb^{(i)}}(\xb_i) ] \cdot y_i \big| \cdot \big\|\nabla_{\Wb_l} f_{\Wb^{(i)}}(\xb_i)\big\|_F \leq \sqrt{m}.
\end{align*}
This completes the proof.
\end{proof}


\section{Experimental Results}\label{section:experimental_results}
In this section we provide numerical calculations of the generalization bounds given by Theorem~\ref{thm:expectederrorbound} and Corollary~\ref{col:expectederrorbound_kernel} on the MNIST dataset \citep{lecun1998gradient}. The main goal of these calculations is to demonstrate that the bounds given in our results are informative and can provide practical insight.

We have done experiments of a five-layer fully connected NN on MNIST dataset (3 versus 8), and calculated the first terms in the bounds given by Theorem~\ref{thm:expectederrorbound} and Corollary~\ref{col:expectederrorbound_kernel}. 
\begin{itemize}[leftmargin= *]
    \item In Figure~\ref{subfig:1}, we plot the first term in the bound of Theorem~\ref{thm:expectederrorbound}  with different values of $R$ and $m$, where the value are approximated by solving the constrained convex optimization problem $\inf_{ f \in \cF( \Wb^{(1)}, R )} \{ (4/n) \cdot \sum_{i=1}^n \ell[y_i\cdot f(\xb_i) ] \}$ with projected stochastic gradient descent.
    \item To demonstrate the scaling of the bound in Corollary~\ref{col:expectederrorbound_kernel}, we calculate the value of $\sqrt{\yb^\top (\bTheta^{(L)})^{-1} \yb /n}$, where $\yb$ is the true label vector with random flips. We plot $\sqrt{\yb^\top (\bTheta^{(L)})^{-1} \yb /n}$ in Figure~\ref{subfig:2} by varying the level of label noise, i.e., ratio of the labels that are flipped. Note that to simplify calculation, we do not consider the $\tilde \yb$ introduced in Corollary~\ref{col:expectederrorbound_kernel}. Clearly, our calculation here gives an upper bound of the generalization bound in Corollary~\ref{col:expectederrorbound_kernel}.
\end{itemize}

\begin{figure}[h]
	\begin{center}
		\begin{tabular}{cc}
			\subfigure[]{\includegraphics[width=0.45\linewidth,angle=0]{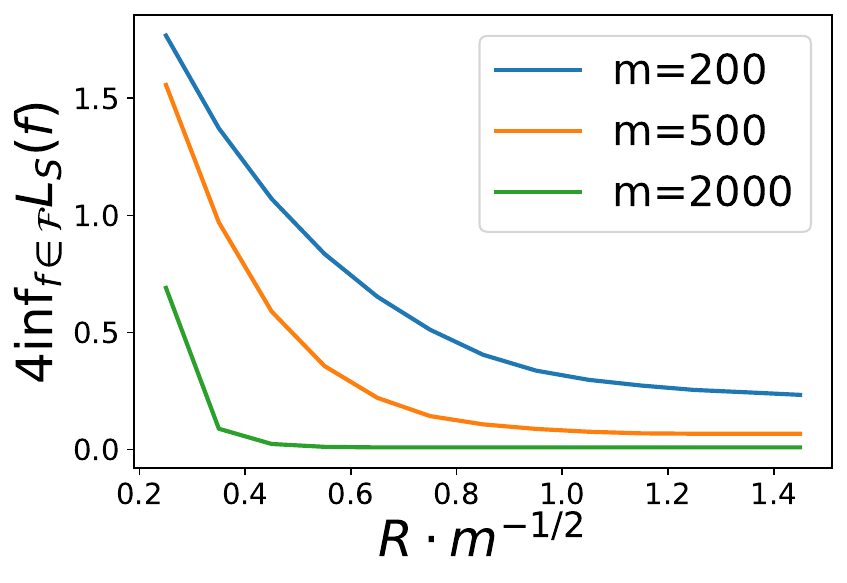}\label{subfig:1}}
			& 
			\subfigure[]{\includegraphics[width=0.465\linewidth,angle=0]{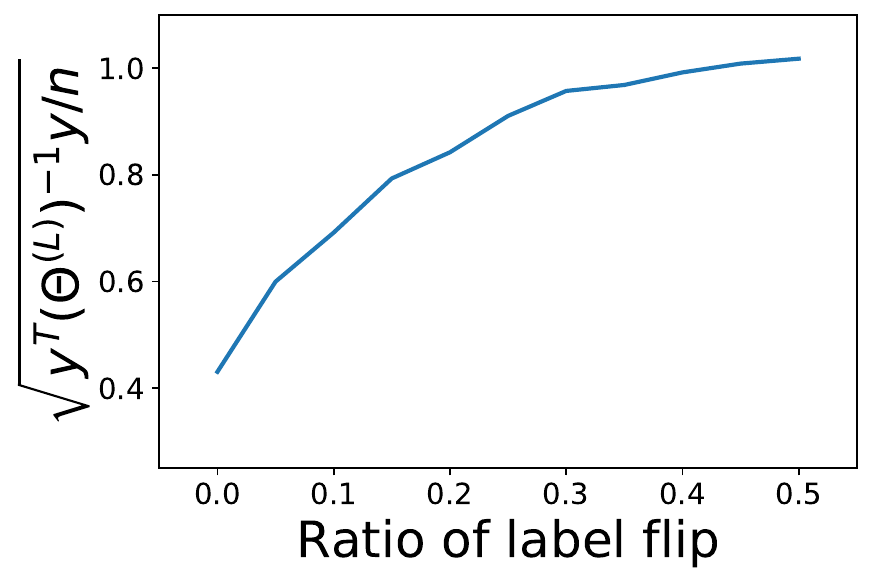}\label{subfig:2}}
		    \end{tabular}
	\end{center}
	\vskip -10pt
	\caption{(a) Evaluation of the first term in the bound of Theorem~\ref{thm:expectederrorbound} for different values of $R$ and $m$. (b) Evaluation of the first term of the bound in Corollary~\ref{col:expectederrorbound_kernel} with different ratio of label flip.} 
	\label{fig:boundevaluation}
\end{figure}

We can see that our bounds in both  Theorem~\ref{thm:expectederrorbound} and Corollary~\ref{col:expectederrorbound_kernel} give small and meaningful values. Moreover, these experimental results also back up our theoretical analysis. In Figure~\ref{subfig:1}, the curves corresponding to different $m$'s also validate our theoretical result that the wider the network is, the shorter SGD needs to travel to fit the training data. In addition, the larger the size of reference function class (i.e., $R$), the smaller $\inf_{ f \in \cF( \Wb^{(1)}, R )} \{ (4/n) \cdot \sum_{i=1}^n \ell[y_i\cdot f(\xb_i) ] \}$ will be. 
In Figure~\ref{subfig:2}, we can see that the noisier the labels, the larger the term $\sqrt{\yb^\top (\bTheta^{(L)})^{-1} \yb /n}$ is. When most of the labels are true labels, our bound can predict good test error; when the labels are purely random (i.e., ratio of label flip  $=0.5$), the bound on the test error can be larger than one. To sum up, these numerical results demonstrate the practical values of our generalization bounds, and suggest that our bounds can provide good measurements of the data classifiability. 

\bibliography{ReLU}
\bibliographystyle{ims}

\end{document}